\documentclass{article}

\usepackage[utf8]{inputenc} 
\usepackage[T1]{fontenc}    
\usepackage{hyperref}       
\usepackage{url}            
\usepackage{booktabs}       
\usepackage{amsfonts}       
\usepackage{nicefrac}       
\usepackage{microtype}      
\usepackage{xcolor}         
\usepackage{url}
\usepackage{amsfonts}
\usepackage{nicefrac}
\usepackage{microtype}      
\usepackage{xcolor}         
\usepackage{graphicx}
\usepackage{booktabs} 
\usepackage{natbib}
\usepackage{booktabs}   \usepackage{mathrsfs}
\usepackage{amsmath,amssymb}
\usepackage{algorithm}
\usepackage{algorithmic}
\usepackage{changepage}
\usepackage{makecell}
\usepackage{microtype}  \usepackage{amsthm}
\usepackage{changepage}
\usepackage{color}
\usepackage{hyperref}
\hypersetup{
  colorlinks,
  citecolor=blue!70!black,
  linkcolor=blue!70!black,
  urlcolor=blue!70!black}
\usepackage{bbm}
\newcommand{\cX}{\mathcal{X}}

\newcommand{\cG}{\mathcal{G}}
\newcommand{\cS}{\mathcal{S}}
\newcommand{\cA}{\mathcal{A}}
\newcommand{\cR}{\mathcal{R}}

\newcommand{\cC}{\mathcal{C}}
\newcommand{\cF}{\mathcal{F}}

\newcommand{\cD}{\mathcal{D}}
\newcommand{\cZ}{\mathcal{Z}}

\newcommand{\cN}{\mathcal{N}}

\newcommand{\EE}{\mathbb{E}}
\newcommand{\RR}{\mathbb{R}}

\newcommand{\ds}{\mathrm{dim}}

\newcommand{\poly}{\mathrm{poly}}
\newcommand{\err}{\mathrm{err}}
\newcommand{\unif}{\mathrm{Unif}}

\newcommand{\one}{\mathbbm{1}}

\newtheorem{assum}{Assumption}

\newtheorem{thm}{Theorem}
\newtheorem{lem}{Lemma}

\newtheorem{defn}{Definition}
\newtheorem{remark}{Remark}
\renewcommand{\epsilon}{\varepsilon}
\renewcommand{\tilde}{\widetilde}
\DeclareMathOperator*{\argmin}{arg\,min}
\DeclareMathOperator*{\argmax}{arg\,max}

\usepackage{fullpage}

\title{Provably Feedback-Efficient Reinforcement Learning via Active 
Reward Learning\footnote{36th Conference on Neural Information Processing Systems (NeurIPS 2022).}
}

\author{
Dingwen Kong\\
Peking University\\
\texttt{dingwenk@pku.edu.cn}\\
\and Lin F. Yang\\
University of California, Los Angles\\
\texttt{linyang@ee.ucla.edu}
}

\begin{document}

\maketitle

\begin{abstract}
An appropriate reward function is of paramount importance in specifying a task in reinforcement learning (RL). Yet, it is known to be extremely challenging in practice to design a correct reward function for even simple tasks. Human-in-the-loop (HiL) RL allows humans to communicate complex goals to the RL agent by providing various types of feedback. However, despite achieving great empirical successes, HiL RL usually requires \emph{too much} feedback from a human teacher and also suffers from insufficient theoretical understanding. In this paper, we focus on addressing this issue from a theoretical perspective, aiming to provide provably feedback-efficient algorithmic frameworks that take human-in-the-loop to specify rewards of given tasks. We provide an \emph{active-learning}-based RL algorithm that first explores the environment without specifying a reward function and then asks a human teacher for only a few queries about the rewards of a task at some state-action pairs. After that, the algorithm guarantees to provide a nearly optimal policy for the task with high probability. We show that, even with the presence of random noise in the feedback, the algorithm only takes $\tilde{O}(H{{\dim_{R}^2}})$ queries on the reward function to provide an $\epsilon$-optimal policy for any $\epsilon > 0$. Here $H$ is the horizon of the RL environment, and $\dim_{R}$ specifies the complexity of the function class representing the reward function. In contrast, standard RL algorithms require to query the reward function for at least $\Omega(\operatorname{poly}(d, 1/\epsilon))$ state-action pairs where $d$ depends on the complexity of the environmental transition.

\end{abstract}

\section{Introduction}





A suitable reward function is essential for specifying a reinforcement learning (RL) agent to perform a complex task. Yet obvious approaches such as hand-designed reward is not scalable for large number of tasks, especially in the multitask settings~\citep{wilson2007multi,brunskill2013sample,yu2020gradient,sodhani2021multi}, and can also be extremely challenging (e.g., \citep{ng1999policy,marthi2007automatic} shows even intuitive reward shaping can lead to undesired side effects). 
Recently, a popular framework called Human-in-the-loop (HiL) RL~\citep{knox2009interactively,christiano2017deep,macglashan2017interactive, ibarz2018reward, lee2021pebble,wang2022skill} gains more interests as it allows humans to communicate complex goals to the RL agent directly by providing various types of feedback. 
In this sense, a reward function can be learned automatically and can also be corrected at proper times if unwanted behavior is happening.
Despite its promising empirical performance, HiL algorithms still suffer from insufficient theoretical understanding and possess drawbacks \cite{arakawa2018dqn}, e.g., it assumes humans can give precise numerical rewards and do so without delay and at every time step, which are usually not true.
Moreover, these approaches usually train on every new task separately and cannot incorporate exiting experiences.

In this paper, we attempt to address the above issues of incorporating humans' feedback in RL from a theoretical perspective. In particular, we would like to address (1) the high \emph{feedback complexity} issue -- i.e., the algorithms in practice usually require large amount feedback from humans to be accurate; (2) feedback from humans can be noisy and non-numerical; (3) in need for support of multiple tasks.
In particular we consider a fixed unknown RL environment, and formulate  a task as an unknown but fixed reward function.
A human who wants the agent to accomplish the task needs to communicate the reward to the agent.
It is not possible to directly specify the parameters of the reward function as the human may not know it exactly as well, but is able to specify good actions at any given state.
To capture the non-numerical feedback issue, we assume that the feedback we can get for an action is only binary -- whether an action is ``good'' or ``bad''. We further assume that the feedback is noisy in the sense that the feedback is only correct with certain probability. 
Lastly, we require that the algorithm, after some initial exploration phase, should be able to accomplish multiple tasks by only querying the reward rather than the environment again. 

In the supervised learning setting, if we only aim to learn a reward function, the feedback complexity can be well-addressed by the active learning framework \cite{settles2009active,hanneke2014theory} -- an algorithm only queries a few samples of the reward entries and then provide a good estimator. 
Yet this become challenging in the RL setting as it is a sequential decision making problem -- state-action pairs that are important in the supervised learning setting may not be accessible in the RL setting.
Therefore, to apply similar ideas in RL, we need a way to explore the environment and collect samples that are important for reward learning. 
Fortunately, there were a number of recent works focusing ``reward-free'' exploration \cite{jin2020reward,wang2020reward} on the environment. Hence, applying such an algorithm would not affect the feedback complexity. 
Additionally it is possible for us to reuse the collected data for multiple tasks.

Our proposed theoretical framework is a non-trivial integration of reward-free reinforcement learning and active learning. The algorithm possesses two phases: in phase I, it performs reward-free RL to explore the environment and collect the small but necessary amount of the information about the environment; in phase II, the algorithm performs active learning to query the human for the reward at only a few state-action pairs and then provide a near optimal policy for the tasks with high probability. 
The algorithm is guaranteed to work even the feedback is noisy and binary and can solve multiple tasks in phase II. Below we summarize our contributions:
\begin{enumerate}
    \item We propose a theoretical framework for incorporating humans' feedback in RL. The framework contains two phases: an unsupervised exploration and an active reward learning phase. Since the two phases are separated, our framework is suitable for multi-task RL.
    \item Our framework deals with a general and realistic case where the human feedback is stochastic and binary-i.e., we only ask the human teacher to specify whether an action is ``good'' or ``bad''. We design an efficient active learning algorithm for learning the reward function from this kind of feedback.
    \item Our query complexity is minimal because it is independent of both the environmental complexity $d$ and target policy accuracy $\epsilon$. In contrast, standard RL algorithms require query the reward function for at least $\Omega(\poly(d, 1/\epsilon))$ state-action pairs. Thus our work provides a theoretical validation for the recent empirical HiL RL works where the number of queries is significantly smaller than the number of environmental steps.
    \item Moreover, we shows the efficacy of our framework in the offline RL setting, where the environmental transition dataset is given beforehand. 
\end{enumerate}

\subsection{Related Work}
\label{related1}
\paragraph{Sample Complexity of Tabular and Linear MDP.}
There is a long line of theoretical work on the sample complexity and regret bound for tabular MDP. See, e.g., \citep{kearns2002near, jaksch2010near, azar2017minimax, jin2018q, zanette2019tighter, agarwal2020model,wang2020long,li2022settling}. The linear MDP is first studied in \citet{yang2019sample}. See, e.g., \citep{yang2020reinforcement,jin2020provably,zanette2020learning,ayoub2020model,zhou2021nearly,zhou2021provably} for sample complexity and regret bound for linear MDP.
\paragraph{Unsupervised Exploration for RL.} The reward-free exploration setting is first studied in \citet{jin2020reward}. This setting is later studied under different function approximation scheme: tabular~\citep{kaufmann2021adaptive,menard2021fast,wu2022gap}, linear function approximation~\citep{wang2020reward,zanette2020provably,zhang2021reward,huang2022towards,wagenmaker2022reward,agarwal2020flambe,modi2021model}, and general function approximation~\citep{qiu2021reward,kong2021online,chen2022statistical}. Besides, \citet{zhang2020task,yin2021optimal} study task-agnostic RL, which is a variety of reward-free RL. \citet{wu2021accommodating} studies multi-objective RL in the reward-free setting. \citet{bai2020provable,liu2021sharp} studies reward-free exploration in Markov games.
\paragraph{Active Learning.}
Active learning is relatively well-studied in the context of unsupervised learning. See, e.g., \citet{dasgupta2007general,balcan2009agnostic,settles2009active,hanneke2014theory} and the references therein. Our active reward learning algorithm is inspired by a line of works~\citep{cesa2009robust,dekel2010robust,agarwal2013selective} considering online classification problem where they assume the response model $P(y|x)$ is linear parameterized. However, their works can not directly apply to the RL setting and also the non-linear case.
There are also many empirical study-focused paper on active reward learning. See, e.g., \citep{daniel2015active,christiano2017deep,sadigh2017active,biyik2019asking,biyik2020active,wilde2020active,lindner2021information,lee2021pebble}. Many of them share similar algorithmic components with ours, like information gain-based active query and unsupervised pre-training. But they do not provide finite query complexity bounds.

\section{Preliminaries}
\subsection{Episodic Markov Decision Process}
In this paper, we consider the finite-horizon Markov decision process (MDP) $M=(\mathcal{S},\mathcal{A},P,r,H,s_{1})$, where $\mathcal{S}$ is the state space, $\mathcal{A}$ is the action space, $P=\{P_h\}_{h=1}^H$ where $P_h:\mathcal{S} \times \mathcal{A} \rightarrow \mathcal{\triangle}(\mathcal{S})$ are the transition operators, $r=\{r_h\}_{h=1}^H$ where $r_h:\mathcal{S} \times \mathcal{A} \rightarrow \{0,1\}$ are the deterministic \emph{binary} reward functions, and $H$ is the planning horizon. Without loss of generality, we assume that the initial state $s_{1}$ is fixed.\footnote{For a general initial distribution $\rho$, we can treat it as the first stage transition probability, $P_1$.}
In RL, an agent interacts with the environment episodically. Each episode consists of $H$ time steps. A deterministic policy $\pi$ chooses an action $a \in \mathcal{A}$ based on the current state $s\in \mathcal{S}$ at each time step $h \in [H]$. Formally, $\pi=\{\pi_h\}_{h=1}^H$ where for each $h \in [H]$, $\pi_h:\mathcal{S}\rightarrow \mathcal{A}$ maps a given state to an action. In each episode, the policy $\pi$ induces a trajectory
$$
s_1,a_1,r_1,s_2,a_2,r_2,...,s_H,a_H,r_H,s_{H+1}
$$
where $s_1$ is fixed, $a_1=\pi_1(s_1)$, $r_1=r_1(s_1,a_1)$, $s_2 \sim P_1(\cdot|s_1,a_1)$, $a_2=\pi_2(s_2)$, etc.

We use Q-function and V-function to evaluate the long-term expected cumulative reward in terms of the current state (state-action pair), and the policy deployed. Concretely, the Q-function and V-function are defined as:
$
Q_h^\pi(s,a)=\mathbb{E}\big[\sum_{h'=h}^Hr_{h'}(s_{h'},a_{h'})|s_h=s,a_h=a,\pi\big]
$
and
$
V_h^\pi(s)=\mathbb{E}\big[\sum_{h'=h}^Hr_{h'}(s_{h'},a_{h'})|s_h=s,\pi\big]
$.
We denote the optimal policy as $\pi^*=\{\pi_h^*\}_{h\in[H]}$,  optimal values as $Q_h^*(s,a)$ and $V_h^*(s)$. Sometimes it is convenient to consider the Q-function and V-function where the true reward function is replaced by a estimated one $\hat{r}=\{\hat{r}_h\}_{h\in[H]}$. We denote them as $Q_h^\pi(s,a,\hat{r})$ and $V_h^\pi(s,\hat{r})$. We also  denote the corresponding optimal policy and value as $\pi^*(\hat{r})$, $Q_h^*(s,a,\hat{r})$ and $V_h^*(s,\hat{r})$.

\paragraph{Additional Notations.} We define the infinity-norm of function $f:\mathcal{S}\times\mathcal{A}\rightarrow \mathbb{R}$ as $$\|f\|_\infty=\sup_{(s,a)\in\mathcal{S}\times\mathcal{A}}|f(s,a)|.$$
For a set of state-action pairs $\mathcal{Z}\subseteq \mathcal{S}\times\mathcal{A}$ and a function $f:\mathcal{S}\times\mathcal{A}\rightarrow \mathbb{R}$, we define $$\|f\|_{\mathcal{Z}}=\left(\sum_{(s,a)\in{\cZ}}f(s,a)^2\right)^{1/2}.$$
\section{Technical Overview}
In this section we give a overview of our learning scenario and notations, as well as the main techniques. The learning process divides into {two} phases.
\subsection{Phase 1: Unsupervised Exploration} 
The first step is to explore the environment without reward signals. Then we can query the human teacher about the reward function in the explored region. We adopt the \emph{reward-free exploration} technique developed in~\citet{jin2020reward, wang2020reward}. The agent is encouraged to do exploration by maximizing the cumulative \emph{exploration bonus}. Concretely, we gather $K$ trajectories $\cD=\{(s_h^k,a_h^k)\}_{(h,k)\in[H]\times[K]}$ by interacting with the environment. We can strategically choose which policy to use. At the beginning of the $k$-th episode, we calculate a policy $\pi_k$ based on the history of the first $k-1$ episodes and use $\pi_k$ to induce a trajectory $\{(s_h^k,a_h^k)\}_{h\in[H]}$.

A similar approach called \emph{unsupervised pre-training}~\citep{sharma2020dynamics,liu2021behavior} has been successfully used in practice. Concretely, in unsupervised pre-training agents are encouraged to do exploration by maximizing various \emph{intrinsic rewards}, such as prediction errors~\citep{houthooft2016vime} and count-based state-novelty~\citep{tang2017exploration}.
\subsection{Phase 2: Active Reward Learning}
The second step is to learn a proper reward function from human feedback. Our work assumes that the underlying valid reward is 1-0 binary, which is interpreted as good action and bad action. We remark that RL problems with binary rewards represent a large group of RL problems that are suitable and relatively easy for having human-in-the-loop. A representative group of problems is the binary judgments: For example, suppose we want a robot to learn to do a backflip. A human teacher will judge whether a flip is successful and assign a reward of 1 for success and 0 for failure. Furthermore, our framework can also be generalized to RL problems with $n$-uniform discrete rewards. The detailed discussion is defered to Appendix~\ref{sec:beyondbinary} due to space limit.

Concretely, consider a fixed stage $h\in[H]$, and we are trying to learn $r_h$ from the human response. Each time we can query a datum $z=(s,a)\in\cD$ and receive an independent random response $Y\in\{0,1\}$ from the human expert, with distribution:
\[
P(Y=1|z)=1-P(Y=0|z)=f_h^{*}(z).
\]
Here $f^*_h$ is the human response model and needs to be learned from data. We assume that the underlying valid reward of $z$ can be determined by $f^*_h(z)$ in the following manner:
\[
\textstyle{
r_h(z)=\left\{\begin{aligned}
    1,\ \ f_h^*(z)>1/2\\
    0,\ \ f_h^*(z)\leq 1/2.
\end{aligned}
\right.}
\]
Note that the query returns 1 with a probability greater than $\frac12$ if and only if the underlying valid reward is 1. 
To make the number of queries  as small as possible, we choose a small subset of informative data to query the human. We adopt ideas in the \emph{pool-based active learning} literature and select informative queries greedily. We show that only $\tilde{O}(H{\dim_{R}^2})$ queries need to be answered by the human teacher. The active query method is widely used in human-involved reward learning in practice and shows superior performance than uniform sampling~\citet{christiano2017deep,ibarz2018reward,lee2021pebble}.

After we learn a proper reward function $\hat{r}$, we use past experience $\cD$ and $\hat{r}$ to plan for a good policy. Note that in this phase we are not allowed for further interaction with the environment. In the multi-task RL setting, we can run Phase 2 for multiple times and reuse the data collected in Phase 1.

Now we discuss the efficacy of our framework. A naive approach for reward learning via human feedback is asking the human teacher to evaluate the reward function in each round. This approach results in equal environmental steps and number of queries. This high query frequency is unacceptable for large-scale problems. For example, in~\citet{lee2021pebble} the agent learns complex tasks with very few queries ($\sim10^2 \text{ to } 10^3$ queries) to the human compared to the number of environmental steps ($\sim10^6$ steps) by utilizing active query technique. From the theoretical perspective, usual RL sample complexity bound scales with $\propto\poly(d,\frac{1}{\epsilon})$, where $d$ is the complexity measure of the environmental transition and $\epsilon$ is the target policy accuracy. This quantity can be huge when the environment is complex (i.e., $d$ is large) or with small target accuracy. Our query complexity is desirable since it is independent of both $d$ and $1/\epsilon$.

\section{Pool-Based Active Reward Learning}
In this section we formally introduce our algorithm for active reward learning. We consider a fixed stage $h$ and learn $r_h$ by querying a small subset of $\cZ_h=\{(s_h^k,a_h^k)\}_{k\in[K]}$. We omit the subscript $h$ in this section, i.e., we use $\cZ,z_k,r,f^*$ to denote $\cZ_h,z_h^k,r_h,f^*_h$ in this section. 
Since $\cZ$ is given before the learning process starts, we refer to this learning scenario as \emph{pool-based} active learning. Our purpose is to learn a reward function $\hat{r}(\cdot)$ such that $r(z)=\hat{r}(z)$ for most of $z$ in $\cZ$. At the same time, we hope the number of queries  can be as small as possible.

We assume $\cF$ is a pre-specified function class to learn $f^*$ from, and $\cF$ is known as a prior. We assume that $\cF$ has enough expressive power to represent the human response. Concretely, we assume the following \emph{realizability}.
\begin{assum}[Realizability]
\label{assum:real}
$f^*\in\cF$.
\end{assum}
The learning problem can be arbitrarily difficult, especially when $f^*(z)$ is close to $\frac12$, in which case it will be difficult to determine the true value of $r(z)$. To give a problem-dependent bound, we assume the following \emph{bounded noise} assumption. In the literature on statistical learning, this assumption is also referred to as \emph{Massart noise}~\citep{massart2006risk,gine2006concentration,hanneke2014theory}. {Our framework can also work under the \emph{low noise} assumption - for brevity, we defer the discussion to Appendix~\ref{sec:beyondbounded}.}
\begin{assum}[Bounded Noise]
\label{assum:margin}
There exists $\Delta>0$, such that for all $z\in\cS\times\cA$,
$$
|f^*(z)-\frac12|>\Delta.
$$
\end{assum}
{The value of the margin $\Delta$ depends on the intrinsic difficulty of the reward learning problem and the capacity of the human teacher. For example, if the reward is rather easy to specify and the human teacher is a field expert, and can always give the right answer with a probability of at least 80\%, then $\Delta$ will be 0.3. But if the learning problem is hard or the human teacher is unfamiliar with the problem and can only give near-random answers, then $\Delta$ will be very small. But in that case, we won’t hope the human teacher can help us in the first place. So a typical good value for $\Delta$ should be a constant.}
\paragraph{Examples.} We give two examples of $\cF$ that is frequently studied in the active learning literature. In the linear model, the function class $\cF$ consists of $f$ in the following form:
$$
f(z)=\frac{\left<\phi(z),w\right>+1}{2}.
$$
In the logistic model, the function class $\cF$ consists of $f$ in the following form:
$$
f(z)=\frac{\exp{\left<\phi(z),w\right>}}{1+\exp{\left<\phi(z),w\right>}}.
$$
Here $\phi:\cS\times\cA\rightarrow \mathbb{R}^d$ is a fixed and known feature extractor, and $w\in\mathbb{R}^d$.

The complexity of $\cF$ essentially depends on the learning complexity of the human response model. We use the following \emph{Eluder dimension}~\citep{russo2014learning} to characterize the complexity of $\cF$. The eluder dimension serves as a common complexity measure of a general non-linear function class in both reinforcement learning literature~\citep{osband2014model,ayoub2020model, wang2020reinforcement,  jin2021bellman} and active learning literature~\citep{chen2021active}.
\begin{defn}[Eluder Dimension]
Let $\varepsilon\ge 0$ and $\cZ =\{(s_i, a_i)\}_{i=1}^n\subseteq \cS\times\cA$ be a sequence of state-action pairs. \\
(1) A state-action pair $(s,a)\in \cS\times \cA$ is \emph{$\varepsilon$-dependent} on $\cZ$ with respect to $\cF$ if any $f, f'\in \cF$ satisfying $\|f - f'\|_{\cZ} \le \varepsilon$ also satisfies $|f(s,a) - f'(s,a)|\le \varepsilon$.  \\
(2) An $(s,a)$ is \emph{$\varepsilon$-independent} of $\cZ$ with respect to $\cF$ if $(s,a)$ is not $\varepsilon$-dependent on $\cZ$. \\
(3) The $\varepsilon$-eluder dimension $\ds_E(\cF,\varepsilon)$ of a function class $\cF$ is the length of the longest sequence of elements in $\cS \times \cA$ such that, for some $\varepsilon' \ge \varepsilon$, every element is $\varepsilon'$-independent of its predecessors.\\
(4) The \emph{eluder dimension} of a function class $\cF$ is defined as
$$
\dim_E(\cF):=\limsup_{\alpha\downarrow0}\frac{\dim_E(\cF,\alpha)}{\log(1/\alpha)}.
$$
\end{defn}

We remark that a wide range of function classes, including linear functions, generalized linear functions and bounded degree polynomials, have bounded eluder dimension. 

\begin{defn}[Covering Number and Kolmogorov Dimension]
\label{def:cover}
For any $\varepsilon >0$, there exists an $\varepsilon$-cover $\cC(\cF, \varepsilon) \subseteq \cF$ with size $|\cC(\cF, \varepsilon)| \le \cN(\cF, \varepsilon)$, such that for any $f\in \cF$, there exists $f'\in \cC(\cF,\varepsilon)$ with  $\|f-f'\|_{\infty}\le \varepsilon$. The Kolmogorov dimension of $\cF$ is defined as:
$$
\dim_{K}(\cF):=\limsup_{\alpha\downarrow0}\frac{\log(\cN(\cF,\alpha))}{\log(1/\alpha)}.
$$
\end{defn}
The Kolmogorov dimension is also bounded by $O(d)$ for linear/generalized linear function class. Throughout this paper, we denote
$$
\dim(\cF):=\max\{\dim_E(\cF),\dim_K(\cF)\} 
$$
as the complexity measure of $\cF$. When $\cF$ is the class of $d$-dimensional linear/generalized linear functions, $\dim(\cF)$ is bounded by $O(d)$.
\subsection{Algorithm}
We describe our algorithm for learning the human response model and the underlying reward function. We sequentially choose which data points to query. Denote $\cZ_k$ the first $k$ points that we decide to query and initial $\cZ_0$ to be an empty set. For each $z\in\cZ$, we use the following bonus function to measure the information gain of querying $z$, i.e., how much new information $z$ contains compared to $\cZ_{k-1}$:
\[
b_k(\cdot)\leftarrow \sup_{f,f'\in\cF,\|f-f'\|_{\cZ_{k-1}}\leq \beta}|f(\cdot)-f'(\cdot)|
\]
We then simply choose $z_k$ to be $\arg\max_{z\in\cZ}b_k(z)$. After the $N$ query points are determined, we query their labels from a human. The human response model is then learned by solving a least-squares regression:
\[
\tilde{f}\leftarrow
\min_{f\in\cF}\sum_{z\in\cZ_N}(f(z)-l(z))^2.
\]
The human response model is used for estimating the underlying reward function. {We round $\tilde{f}$ to the cover $\cC(\cF,\Delta/2)$ to ensure that there are a finite number of possibilities of such functions – this gives us the convenience of applying union bound in our analysis. Indeed, we believe a more refined analysis would remove the requirement of rounding but will make the analysis much more involved.} The whole algorithm is presented in Algorithm~\ref{alg:active_reward_learning}. Note that such an interactive mode with the human teacher is \emph{non-adaptive} since all queries are given to the human teacher in one batch. This property makes our algorithm desirable in practice. {Here we assume the value of $\Delta$ is known as a prior. We can extend our results to the case where $\Delta$ is unknown. For brevity, we defer the discussion to Appendix~\ref{sec:unknown}.}
\begin{algorithm}[hbt]
	\caption{Active Reward Learning($\cZ$, $\Delta$, $\delta$)\label{alg:active_reward_learning}}
	\begin{algorithmic}
		\STATE \textbf{Input:} Data Pool $\cZ=\{z_i\}_{i\in[T]}$, margin $\Delta$, failure probability $\delta\in(0,1)$
		\STATE $\cZ_0\leftarrow \{\}$ {\color{blue}//Query Dataset}
		\STATE Set
		$
		N\leftarrow C_1\cdot \frac{(\dim^2(\cF)+\dim(\cF)\cdot \log(1/\delta))\cdot (\log^2(\dim(\cF)))}{\Delta^2}
		$
	    \FOR{$k=1,2,...,N$}
		\STATE $\beta\leftarrow C_2\cdot \sqrt{\log(1/\delta)+\log N\cdot\dim(\cF)}$
		\STATE Set the bonus function: 
		$$b_k(\cdot)\leftarrow \sup_{f,f'\in\cF,\|f-f'\|_{\cZ_{k-1}}\leq \beta}|f(\cdot)-f'(\cdot)|
		$$
		
		\STATE $z_k\leftarrow \arg\max_{z\in\cZ}b_k(z)$
		\STATE  $\cZ_{k}\leftarrow\cZ_{k-1}\cup\{z_k\}$
		\ENDFOR
		\FOR{$z\in\cZ_N$}
		\STATE Ask the human expert for a label $l(z)\in\{0,1\}$
		\ENDFOR
		\STATE Estimate the human model as
		$$	\tilde{f}=\arg\min_{f\in\cF}\sum_{z\in\cZ_N}(f(z)-l(z))^2
		$$
		\STATE Let $\hat{f}\in\cC(\cF,\Delta/2)$ such that $\|\hat{f}-\tilde{f}\|_{\infty}\leq \Delta/2$
		\STATE Estimate the underlying true reward:
		$
		\hat{r}(\cdot)=\left\{\begin{aligned}
    1,\ \ \hat{f}(\cdot)>1/2\\
    0,\ \ \hat{f}(\cdot)\leq 1/2
\end{aligned}
\right.
		$
		\STATE \textbf{return:} The estimated reward function $\hat{r}$.
	\end{algorithmic}
\end{algorithm}

\subsection{Theoretical Guarantee}
\begin{thm}
\label{thm:arl}
With probability at least $1-\delta$, for all $z\in\cZ$,
 we have, $
\hat{r}(z)=r(z).
$
The total number of queries is bounded by
$$
O\left(\frac{(\dim^2(\cF)+\dim(\cF)\cdot \log(1/\delta))\cdot (\log^2(\dim(\cF)))}{\Delta^2}\right).
$$
\end{thm}
\vspace{-2mm}
\paragraph{Proof Sketch} The first step is to show that the sum of bonus functions $\sum_{k=1}^K b_k(z_k)$ is bounded by $O(d\sqrt{K})$ using ideas in \citet{russo2014learning}. Note that the bonus function $b_k(\cdot)$ is non-increasing. Thus we can show that after selecting $N=\widetilde{O}(\frac{d^2}{\Delta^2})$ points, for all $z\in\cZ$, the bonus function of $z$ does not exceed $\Delta$. By the {bounded}-noise assumption, we know that the reward label for $z$ is correct for all $z$.

\section{Online RL with Active Reward Learning}
In this section we consider how to apply active reward learning method in the online RL setting. In this setting the agent is allowed to actively explore the environment without reward signal in the exploration phase. We consider both tabular MDP and linear MDP cases. 
\subsection{Linear MDP with Positive Features}
The linear MDP assumption was first studied in \citet{yang2019sample} and then applied in the online setting \cite{jin2020provably}. It is assumed that the agent is given a feature extractor $\phi:\cS\times\cA\rightarrow\RR^d$ and the transition model can be predicted by linear functions of the give feature extractor. In our work, we additionally assume that the coordinates of $\phi$ are all \emph{positive}. This assumption essentially reduce the linear MDP model to the soft state aggregation model~\citep{singh1994reinforcement,duan2019state}. As will be seen later, the latent state structure helps the learned reward function to generalize.
\begin{assum}[Linear MDP with Non-Negative Features]
\label{assum:linear}
For all $h\in[H]$, we assume that there exists a function $\mu_h:\cS\rightarrow \RR^d$ such that $P_{h}(s'|s,a)=\left<\mu_{h}(s'),\phi(s,a)\right>$. Moreover, for all $(s,a)\in\cS\times\cA$, the coordinates of $\phi(s,a)$ and $\mu_h(s,a)$ are all non-negative.
\end{assum}

\subsection{Exploration Phase}
In the exploration phase, inspired by former works on reward-free RL, we use optimistic least-squares value iteration (LSVI) based algorithm with zero reward. In the linear case, for any $V:\cS\rightarrow \RR$, we estimate $P_h V$ in the following manner
\begin{align}
\label{eqn:lsvi}
\widehat{P}_h^k {V}  (\cdot,\cdot)\leftarrow {w}^T \phi(\cdot,\cdot), \text{ where }{w}\leftarrow \argmin_{w\in\RR^d} \sum_{\tau=1}^{k-1}(w^T \phi(s_h^\tau,a_h^\tau)-V(s_{h+1}^\tau))^2+\|w\|_2^2.
\end{align}
For the tabular case, we simply use the empirical estimation of $P_h$:
\begin{align}
\label{eqn:lsvitbl}
\widehat{{P}}^k_h(s'|s,a)=\left\{\begin{aligned}
&\textstyle{\frac{N_h^k(s,a,s')}{N_h^k(s,a)}},\quad &N_h^k(s,a)>0\\
&\textstyle{\frac{1}{S}},\quad &N_h^k(s,a)=0
\end{aligned}
\right.
\end{align}
and define $\widehat{P}_h^k V$ in the conventional manner. Here $$N_h^k(s,a,s')=\sum_{\tau=1}^{k-1}\mathbbm{1}\{(s_h^{\tau},a_h^{\tau},s_{h+1}^{\tau})=(s,a,s')\}$$ and  $$N_h^k(s,a)=\sum_{\tau=1}^{k-1}\mathbbm{1}\{(s_h^{\tau},a_h^{\tau})=(s,a)\}$$ are the numbers of visit time.

The following \emph{optimism bonus} $\Gamma_h^k(\cdot,\cdot)$ is sufficient to guarantee optimism in standard regret minimization RL algorithms. (The choices of $\beta_{\text{tbl}}$ and $\beta_{\text{lin}}$ is specified in the appendix)
\begin{align}
\label{eqn:opt_bonus}
\Gamma_h^k(\cdot,\cdot)\leftarrow
\left\{
\begin{aligned}
&\min\{\beta_{\text{lin}}\cdot(\phi(\cdot,\cdot)^T(\Lambda^k_h)^{-1}\phi(\cdot,\cdot))^{1/2},H\}, \ \ &\text{(Linear Case)}\\
&\min\{\beta_{\text{tbl}}\cdot N_h^k(\cdot,\cdot)^{-1/2},H\}, \ \ &\text{(Tabular Case)}.
\end{aligned}
\right.
\end{align}
In our setting we enlarge the optimism bonus to the following \emph{exploration bonus}. 
\begin{equation}
\label{eqn:explore_bonus}
b_h^k(\cdot,\cdot)\leftarrow
\left\{
\begin{aligned}
&3\Gamma(\cdot,\cdot), \ \ &\text{(Linear Case)}\\
&C\cdot\textstyle{\frac{H^2 S}{N_h^k(\cdot,\cdot)}}+2\Gamma_h^k(\cdot,\cdot), \ \ &\text{(Tabular Case)}.
\end{aligned}
\right.
\end{equation}
We then set the optimistic Q-function as \[\overline{Q}^k_h(\cdot,\cdot)\leftarrow\Pi_{[0,H-h+1]}[ \widehat{P}^k_h\overline{V}_{h+1}^k (\cdot,\cdot)+b_h^k(\cdot,\cdot)]\]
and define the exploration policy as the greedy policy with respect to $\overline{Q}^k_h$.
{\small
\begin{algorithm}[h]
	\caption{UCBVI-Exploration \label{alg:explore}}
	\begin{algorithmic}
		\FOR{$k=1,2,...,K$}
		\STATE $\overline{V}^k_{H+1}\leftarrow 0$, $\overline{Q}^k_{H+1}\leftarrow 0$
		\FOR{$h=H,H-1,...,1$}
		\STATE Estimate $\widehat{P}_h^k\overline{V}_{h+1}^k (\cdot,\cdot)$ using \eqref{eqn:lsvi} or \eqref{eqn:lsvitbl}
		\STATE Set the optimism bonus $\Gamma_h^k(\cdot,\cdot)$ using \eqref{eqn:opt_bonus} 
		\STATE Set the exploration bonus $b_h^k(\cdot,\cdot)$ using \eqref{eqn:explore_bonus}. 
		\STATE Set the optimistic Q-function $\overline{Q}^k_h(\cdot,\cdot)\leftarrow\Pi_{[0,H-h+1]}[ \widehat{P}^k_h\overline{V}_{h+1}^k (\cdot,\cdot)+b_h^k(\cdot,\cdot)]$
		\STATE ${\pi}^k_h(\cdot)\leftarrow \argmax_{a\in\cA}\overline{Q}^k_h(\cdot,a)$
		\STATE $\overline{V}^k_h(\cdot)\leftarrow \max_{a\in\cA}\overline{Q}^k_h(\cdot,a)$
		\ENDFOR
		\STATE Execute policy $\pi^{k}=\{\pi_h^k\}_{h\in[H]}$ to induce a trajectory $s_1^k,a_1^k,...,s^k_H,a^k_H,s^k_{H+1}$.
		\ENDFOR
        \STATE \textbf{return:} Dataset $\cD=\{(s_h^k,a_h^k)\}_{(h,k)\in[H]\times[K]}$
	\end{algorithmic}
\end{algorithm}
}
\subsection{Reward Learning \& Planning Phase}
After the exploration phase, we run the active reward learning algorithm introduced before on the collected dataset. In the linear setting, we replace the original action with uniform random action. We then use the learned reward function to plan for a near-optimal policy. We still add optimism bonus to guarantee optimism. The whole algorithm is presented in Algorithm~\ref{alg:plan}
{\small
\begin{algorithm}[h]
	\caption{UCBVI-Planning \label{alg:plan}}
	\begin{algorithmic}
		\STATE \textbf{Input:} Dataset $\cD=\{(s_h^k,a_h^k)\}_{(h,k)\in[H]\times[K]}$
		\FOR{$h=1,2,...,H$}
		\IF{Linear Case}
        \STATE $\widetilde{\cZ}_h\leftarrow\{(s_h^k,\tilde{a}_h^k)\}_{k\in[K]}$, where $\{\tilde{a}^{k}_{h}\}_{k\in[K]}$ are sampled i.i.d. from $\unif(\cA)$
        \STATE  $\hat{r}_h\leftarrow \text{Active Reward Learning}(\widetilde{\cZ}_h,\Delta,\delta/(2H))$.
        \ELSIF{Tabular Case}
        \STATE ${\cZ}_h\leftarrow\{(s_h^k,{a}_h^k)\}_{k\in[K]}$
        \STATE  $\hat{r}_h\leftarrow \text{Active Reward Learning}({\cZ}_h,\Delta,\delta/(2H))$.
        \ENDIF
		\ENDFOR
		\FOR{$k=1,2,...,K$}
		\STATE ${V}^k_{H+1}\leftarrow 0$, ${Q}^k_{H+1}\leftarrow 0$
		\FOR{$h=H,H-1,...,1$}
		\STATE Estimate $\widehat{P}_h^k {V}_{h+1}^k (\cdot,\cdot)$ using \eqref{eqn:lsvi} or \eqref{eqn:lsvitbl}
		\STATE Set the optimism bonus $\Gamma_h^k(\cdot,\cdot)$ using \eqref{eqn:opt_bonus} 
		\STATE Set the optimistic Q-function ${Q}^k_h(\cdot,\cdot)\leftarrow\Pi_{[0,H-h+1]}[\hat{r}_h(\cdot,\cdot) + \widehat{P}_h^k {V}_{h+1}^k+\Gamma_h^k(\cdot,\cdot)]$
		\STATE $\hat{\pi}_h^k(\cdot)\leftarrow \argmax_{a\in\cA}{Q}^k_h(\cdot,a)$
		\STATE ${V}_h^k(\cdot)\leftarrow \max_{a\in\cA}{Q}_h^k(\cdot,a)$
		\ENDFOR
		\ENDFOR
		\STATE \textbf{return:} $\hat{\pi} $ drawn uniformly from $\{\hat{\pi}^k\}_{k=1}^{K}$ where $\hat{\pi}^k=\{\hat{\pi}^k_h\}_{h\in[H]}$
	\end{algorithmic}
\end{algorithm}
}
\subsection{Theoretical Guarantee}
\begin{thm}
\label{thm:main}
In the linear case, our algorithm can find an $\epsilon$-optimal policy with probability at least $1-\delta$, with at most
\[
{O}\left(\frac{|\cA|^2d^5\dim^3(\cF)H^4\iota^3}{\epsilon^2}\right),\ \ \iota=\log(\frac{HSA}{\epsilon\delta\Delta})
\]
episodes. In the tabular case, our algorithm can find an $\epsilon$-optimal policy with probability at least $1-\delta$, with at most
\[
O\left(\frac{H^4SA\iota}{\epsilon^2}+\frac{H^3S^2A\iota^2}{\epsilon}\right), \ \ \iota=\log(\frac{HSA}{\epsilon\delta})
\]
episodes. In both cases, the total number of queries to the reward is bounded by $\tilde{O}(H\cdot\dim^2(\cF)/\Delta^2)$.
\end{thm}
\begin{remark}
Theorem~\ref{thm:main} readily extends to multi-task RL setting by replacing $\delta$ with $\delta/N$ and applying a union bound over all tasks, where $N$ is the number of tasks. The corresponding sample complexity bound only increase by a factor of $\poly\log(N)$.
\end{remark}
\begin{remark}
Standard RL algorithms require to query the reward function for at least ${\Omega}(\frac{\max\{\dim_R, \dim_P\}^2}{\epsilon^2})$ times, where $\dim_R$ and $\dim_P$ stand for the complexity of the reward/transition function. See, e.g., \citet{jin2020provably,zanette2020learning, wang2020reinforcement} for the derivation of this bound. Compared to this bound, our feedback complexity bound has two merits: 1) In practice the transition function is generally more complex than the reward function, thus $\max\{\dim_R, \dim_P\}\gg\dim_R$; 2) Our bound is independent of $\epsilon$ - note that $\epsilon$ can be arbitrarily small, whereas $\Delta$ is a constant.
\end{remark}
\paragraph{Proof Sketch} The suboptimality of the policy $\hat{\pi}$ can be decomposed into two parts:
\[
V_1^{\pi^*}-V_1^{\hat{\pi}}\leq \underbrace{|V_1^{{\pi}^*(\hat{r})}(\hat{r})-V_1^{\hat{\pi}}(\hat{r})|}_{(i)}+\underbrace{|V_1^{\pi^*}-V_1^{\pi^*}(\hat{r})| + |V_1^{\hat{\pi}}-V_1^{\hat{\pi}}(\hat{r})|}_{(ii)}
\]
where (i) correspond to the planning error in the planning phase (Algorithm 3) and (ii) correspond to the estimation error of the reward $\hat{r}$. By standard techniques from the reward-free RL, (i) can be upper bounded by the expected summation of the exploration bonuses in the exploration phase. In order to bound (ii), we need the learned reward function to be \emph{universally} correct, not just on the explored region. We show that the dataset collected in the exploration phase essentially \emph{cover} the state space (tabular case) or the latent state space. Since the reward function class has bounded complexity ($\log|\cR|$ is bounded due to the bounded covering number of $\cF$), the reward function learned from the exploratory dataset can generalized to a distribution induced by any policy.
\section{Offline RL with Active Reward Learning}
In this section we consider the \emph{offline RL} setting, where the dataset $\cD$ is provided beforehand. We show that our active reward learning algorithm can still work well in this setting. In order to give meaningful result, we assume the following \emph{compliance} property of $\cD$ with respect to the underlying MDP. This assumption is firstly introduced in \citet{jin2021pessimism}. Unlike many literature for offline RL, we do not require strong coverage assumptions, e.g., concentratability~\citep{szepesvari2005finite,antos2008learning,chen2019information}. 
\begin{defn}[Compliance]
For a dataset $\cD=\{(s_h^k,a_h^k)\}_{(h,k)\in[H]\times[K]}$, let $\mathbb{P}_{\cD}$ be the joint distribution of the data collecting process. We say $\cD$ is compliant with the underlying MDP if
$$
\mathbb{P}_{\cD}(s_{h+1}^k=s|\{(s_h^j,a_h^j)\}_{j=1}^k,\{s_{h+1}^j\}_{j=1}^{k-1})=P_h(s|s_h^k,a_h^k)
$$
holds for all $h\in[H],k\in[K],s\in\cS$.
\end{defn}

\subsection{Algorithm and Theoretical Guarantee}
At the beginning of the algorithm we call the active reward learning algorithm to estimate the reward function. Inspired by \citet{jin2021pessimism}, we estimated the optimal Q-value $Q^k$ using \emph{pessimistic} value iteration with empirical transition and learned reward function. The policy is defined as the greedy policy with respect to $Q^k$. The full algorithm and theoretical guarantee is stated below.
\begin{algorithm}[htb]
	\caption{LCBVI-Tabular-Offline \label{alg:tabularpess}}
	\begin{algorithmic}
		\STATE \textbf{Input:} Dataset $\cD=\{(s_h^k,a_h^k)\}_{(h,k)\in[H]\times[K]}$
		\FOR{$h=1,2,...,H$}
        \STATE $\cZ_h\leftarrow\{(s_h^k,a_h^k)\}_{k\in[K]}$
        \STATE  $\hat{r}_h\leftarrow \text{Active Reward Learning}(\cZ_h,\Delta,\delta/(2H))$.
		\ENDFOR
		\STATE $\widehat{V}_{H+1}\leftarrow 0$.
		\FOR{$h=H,H-1,...,1$}
		\STATE $\Gamma_h(\cdot,\cdot)\leftarrow \beta'_{\text{tbl}}\cdot(N_h(\cdot,\cdot)+1)^{-1/2}$
		\STATE  ${Q}_h(\cdot,\cdot)\leftarrow\Pi_{[0,H-h+1]}[\hat{r}_h(\cdot,\cdot) + \widehat{\mathbb{P}}_h \widehat{V}_{h+1}(\cdot,\cdot)-2\Gamma_h(\cdot,\cdot)]$
		\STATE $\hat{\pi}_h(\cdot)\leftarrow \argmax_{a\in\cA}{Q}_h(\cdot,a)$
		\STATE ${V}_h(\cdot)\leftarrow \max_{a\in\cA}{Q}_h(\cdot,a)$
		\ENDFOR
		\STATE \textbf{return:} $\hat{\pi}=\{\hat{\pi}_h\}_{h\in[H]}$
	\end{algorithmic}
\end{algorithm}

\begin{thm}
\label{thm:pess}
With probability at least $1-\delta$, the sub-optimal gap of $\hat{\pi}$ is bounded by
\[
V_1^*(s_1)-V_1^{\hat{\pi}}(s_1)\leq  2\left( H\sqrt{S\log(SAHK/\delta)}\cdot\mathbb{E}_{\pi^*}\left[\sum^H_{h=1}(N_h(s_h,a_h)+1)^{-1/2}\right]\right).
\]
And the total number of queries is bounded by $\tilde{O}(H\cdot\dim^2(\cF)/\Delta^2)$.
\end{thm}
The proof of Theorem~\ref{thm:pess} is deferred to the appendix.

\section{Numerical Simulations} We run a few experiments to test the efficacy of our algorithmic framework and verify our theory. We consider a tabular MDP with linear reward. The details of the experiments are deferred to Appendix~\ref{experiment}.
Here we highlight three main points derived from the experiment.
\begin{itemize}\item Active learning helps to reduce feedback complexity compared to passive learning. For instance, to learn a $0.02$-optimal policy, the active learning-based algorithm only needs $\sim 70$ queries to the human teacher, while the passive learning-based algorithm requires $\sim 200$ queries. (Figure~\ref{fig},  left panel)
\item The noise parameter $\Delta$ plays an essential role in the feedback complexity, which is consistent with our bound. For instance, with fixed number of queries, the average error of the learned policy is $0.05, 0.02, 0.005$ for $\Delta=0.02,0.05,0.1$. (Figure~\ref{fig}, right panel)
\item When $\Delta$ is relatively large (which indicates that the reward learning problem is not inherently difficult for the human teacher), we can learn an accurate policy with much fewer queries to the human teacher compared to the number of environmental steps. For instance, for $\Delta=0.05$, to learn a $0.01$-optimal policy, our algorithm requires $\sim 2000$ environmental steps but only requires $\sim 150$ queries. (Figure~\ref{fig}, left panel)\end{itemize}
\begin{figure}[hbt]
  \centering
  \includegraphics[scale=0.35]{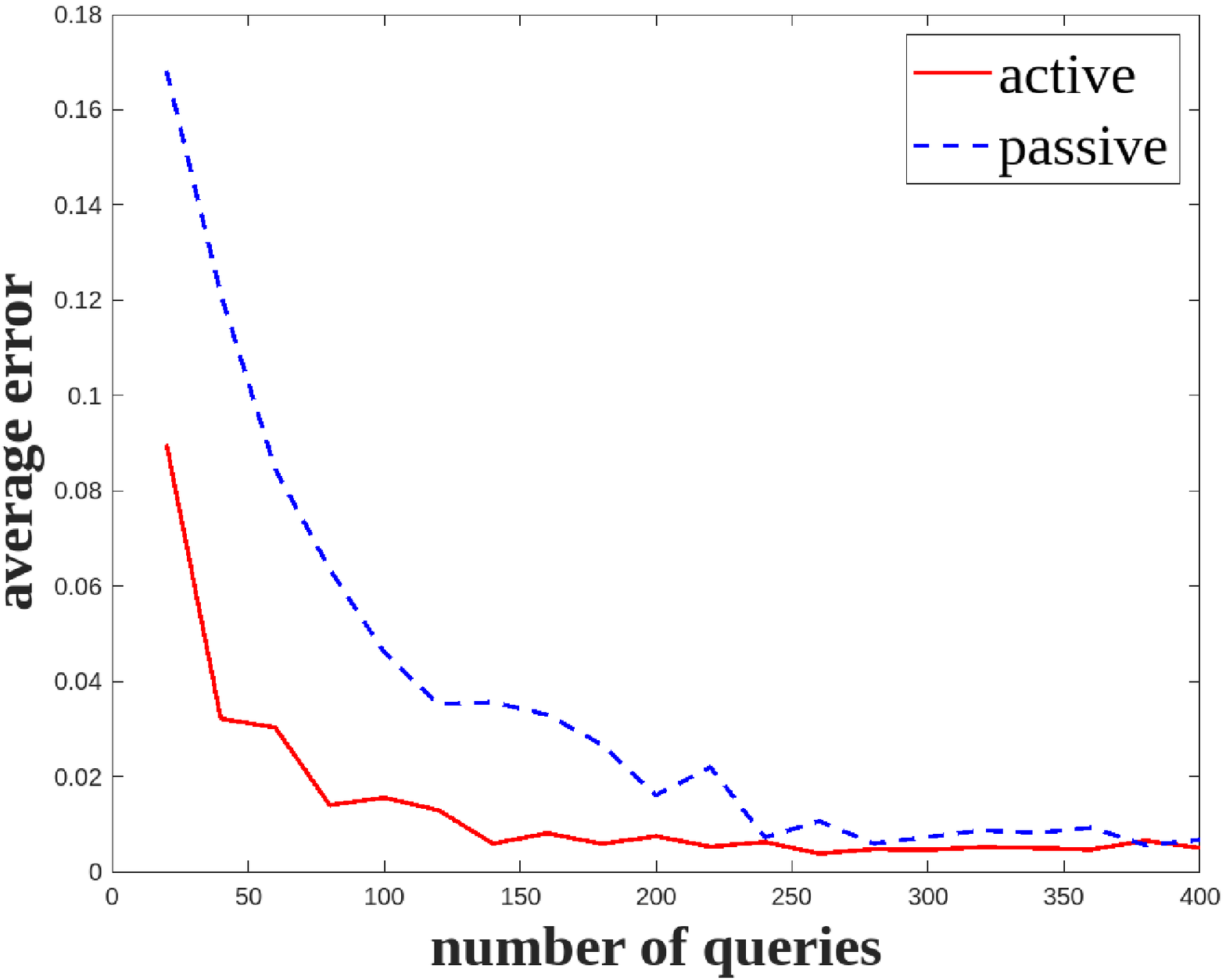}~~
    \includegraphics[scale=0.35]{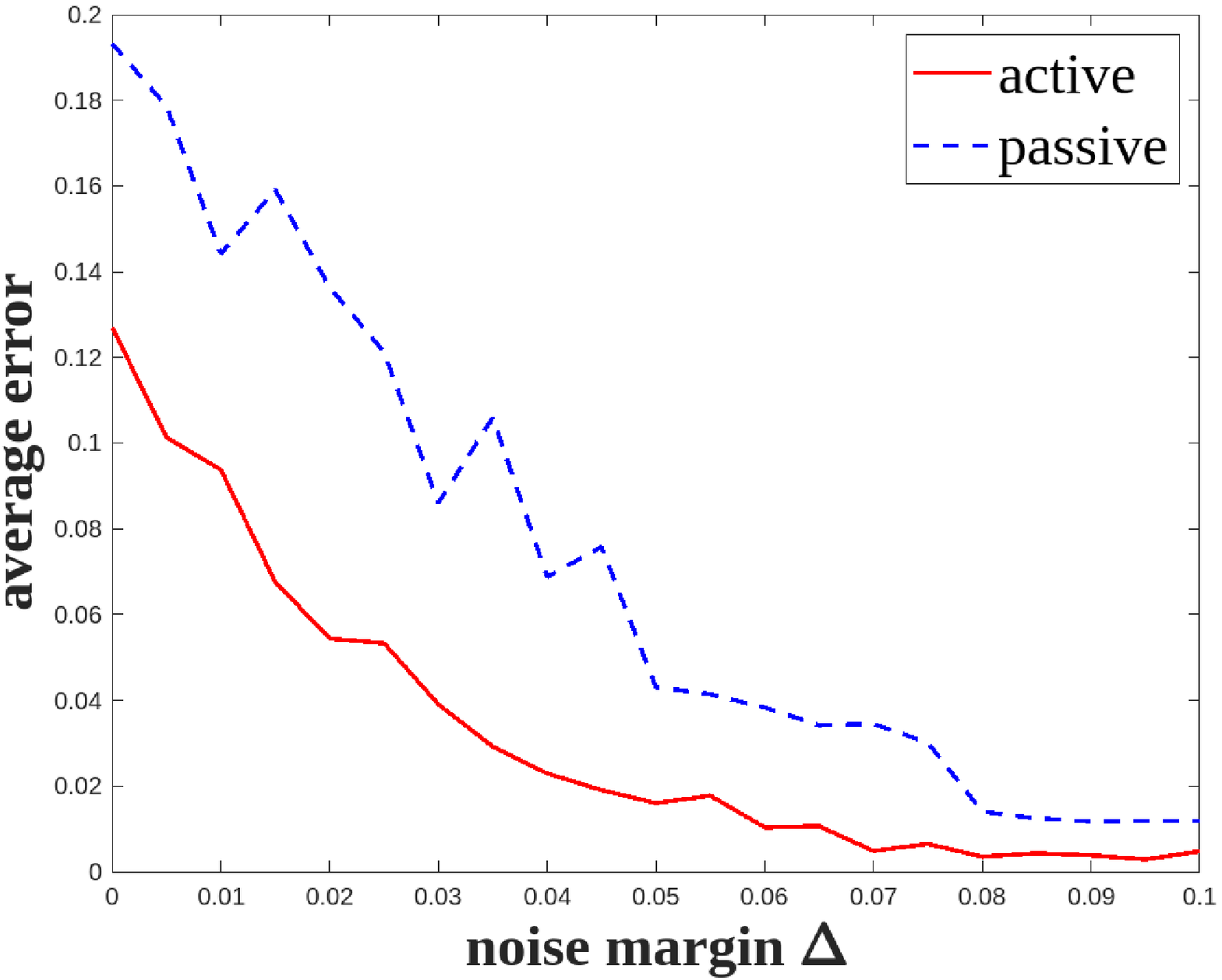}
  \caption{\emph{Left}: average error v.s. number of queries. \emph{Right}: the effect of the noise margin $\Delta$. \label{fig}}
\end{figure}
\section{Conclusions and Discussions}
\label{sec:con}
In this work, we provide a provably feedback-efficient algorithmic framework that takes human-in-the-loop to specify rewards of given tasks. Our proposed framework theoretically addresses several issues of incorporating humans' feedback in RL, such as noisy, non-numerical feedback and high feedback complexity. Technically, our work integrates reward-free RL and active learning in a non-trivial way. {The current framework is limited to information gain-based active learning}, and an interesting future direction is incorporating different active learning methods, such as disagreement-based active learning, into our framework.

{From a broad perspective, our work is a theoretical validation of recent empirical successes in HiL RL. Our results also brings new ideas to practice: it provides a new type of selection criterion that can be used in active queries; it suggests that one can use recently developed reward-free RL algorithms for unsupervised pre-training. These ideas can be combined with existing deep RL frameworks to be scalable. A limitation of the current work is that it mainly focus on theory, and we leave the empirical test of these ideas in real-world deep RL as future work. }
\section*{Acknowledgement}
DK is partially supported by the elite undergraduate training program of School of Mathematical Sciences in Peking University. LY is supported in part by DARPA grant HR00112190130, NSF Award 2221871.  

\bibliography{ref}
\bibliographystyle{plainnat}



\newpage
\appendix
\paragraph{Road map for the appendices } In Section~\ref{experiment} we provide the numerical simulation results. From Section~\ref{proof1} to Section~\ref{proof3}, we give proofs of Theorem~\ref{thm:arl} to Theorem~\ref{thm:pess}.
\section{Numerical Simulations}
\label{experiment}
We run a few experiments to test the efficacy of our algorithmic framework and verify our theory. We simulate a random two-stage ($H=2$) tabular MDP with $S=20$, $A=10$ (stage 1) or $3$ (stage 2). We consider the linear response model 
\[
f(z)=\frac{\left<\phi(z),w\right>+1}{2}
\]
where $\phi(z),w\in\RR^d$ and $d=5$. We fix the environmental steps in the exploration phase to be $K=2000$. In our environment, the noise margin parameter $\Delta$ (defined in Assumption~\ref{assum:margin}) can be adjusted by removing states that do not satisfy the assumption. The error is defined as \[V_1^*(s_1)-V_1^{\hat{\pi}}(s_1)\] where $\hat{\pi}$ is the learned policy and $s_1$ is the fixed initial state.
\paragraph{Active Learning v.s. Passive Learning.}
The left panal of Figure~\ref{fig} shows a comparison between the algorithm with active reward learning (our Algorithm~\ref{alg:active_reward_learning}) and the algorithm with passive learning. The only difference is that instead of actively choosing queries, the passive learning algorithm uniformly samples queries from the dataset collected in the exploration phase. The left panel of Figure~\ref{fig} shows that the active reward learning method significantly reduces the number of queries needed to achieve a target policy accuracy. In this experiment, the noise margin parameter is $\Delta=0.05$.
\paragraph{The Effect of the Noise Margin $\Delta$.} Theorem~\ref{thm:arl} suggests that the noise margin $\Delta$ will significantly influence the difficulty of the reward learning problem. The right panel of Figure~\ref{fig} verifies this effect when the number of queries is fixed to $N=100$. Moreover, we also compare between active learning and passive learning in this setting.
\paragraph{Implementation Details.} The transition probabilities of the MDP are generated uniformly from the $S$-dimensional probability simplex. The features in the response model are generated from a uniform ball distribution with random scaling. Results in Figure~\ref{fig} are averaged over $100$ trials. We use a few MATLAB package that are listed below. The source code is given is the supplementary material. One may run Figure1.m and Figure2.m to reproduce the results in Figure~\ref{fig}

Dahua Lin (2022). Sampling from a discrete distribution (\href{https://www.mathworks.com/matlabcentral/fileexchange/21912-sampling-from-a-discrete-distribution}{link}), MATLAB Central File Exchange. Retrieved May 26, 2022.

David (2022). Uniform Spherical Distribution Generator (\href{https://www.mathworks.com/matlabcentral/fileexchange/67384-uniform-spherical-distribution-generator}{link}), MATLAB Central File Exchange. Retrieved May 26, 2022.

Roger Stafford (2022). Random Vectors with Fixed Sum (\href{https://www.mathworks.com/matlabcentral/fileexchange/9700-random-vectors-with-fixed-sum}{link}), MATLAB Central File Exchange. Retrieved May 26, 2022.
\section{Proof of Theorem 1}
\label{proof1}
The next lemma bound the error of the regression.
\begin{lem}
\label{lem:logistic}
With probability at least $1-\delta$, 
\[
\|f^*-\tilde{f}\|_{\cZ_N}\leq O(\sqrt{\log(1/\delta)+\log N\cdot\dim(\cF)})
\]
\end{lem}
\begin{proof}
For any $f\in\cF$ and $z\in\cZ_N$, consider
\[
\xi(z,f)=2(f(z)-f^*(z))(f^*(z)-l(z))
\]
where $l(z)$ is the response from the human expert. Note that
\[
\mathbb{E}[\xi(z,f)]=0,\ \ |\xi(z,f)|\leq 2|f(z)-f^*(z)|.
\]
By Hoeffding's inequality, for a fixed $f\in\cF$, we have that
\[
\Pr\left[\left|\sum_{z\in\cZ_N}\xi(z,f)\right|\geq \epsilon\right]\leq 2\exp\left(-\frac{\epsilon^2}{8\|f-f^*\|_{\cZ_N}^2}\right)
\]
Let
\begin{align*}
    \epsilon&=\left(8\|f-f^*\|^2_{\cZ_N}\log\left(\frac{2\cN(\cF,1/N)}{\delta}\right)\right)^{\frac12}\\
    &\leq 4\|f-f^*\|_{\cZ_N}\cdot\sqrt{\log(2/\delta)+\log(\cN(\cF,1/N))}.
\end{align*}
We have that with probability at least $1-\delta$, for all $f\in\cN(\cF,1/N)$,
\[
\left|\sum_{z\in\cZ_N}\xi(z,f)\right|\leq 4\|f-f^*\|_{\cZ_N}\cdot\sqrt{\log(2/\delta)+\log(\cN(\cF,1/N))}
\]
Condition on the above event for the rest of the proof. Consider any $f\in \cF$, there exists $g\in\cN(\cF,1/N)$ such that
$\|g-f\|_{\infty}\leq 1/N$. Thus we have that
\begin{align*}
\left|\sum_{z\in\cZ_N}\xi(z,f)\right|&\leq\left|\sum_{z\in\cZ_N}\xi(z,g)\right|+2N\cdot\frac{1}{N}\\
&\leq4\|g-f^*\|_{\cZ_N}\cdot\sqrt{\log(2/\delta)+\log(\cN(\cF,1/N))}+2\\
&\leq4(\|f-f^*\|_{\cZ_N}+1)\cdot\sqrt{\log(2/\delta)+\log(\cN(\cF,1/N))}+2.
\end{align*}
In particular,
\[
\left|\sum_{z\in\cZ_N}\xi(z,\hat{f})\right|\leq 4(\|\hat{f}-f^*\|_{\cZ_N}+1)\cdot\sqrt{\log(2/\delta)+\log(\cN(\cF,1/N))}+2.
\]
On the other hand,
\begin{align*}
\sum_{z\in\cZ_N}\xi(z,\hat{f})&=\|\hat{f}-l\|^2_{\cZ_N}-\|\hat{f}-f^*\|^2_{\cZ_N}-\|f^*-l\|^2_{\cZ_N}\\
&\leq -\|\hat{f}-f^*\|^2_{\cZ_N}
\end{align*}
Thus we have
\[
\|\hat{f}-f^*\|^2_{\cZ_N} \leq 4(\|\hat{f}-f^*\|_{\cZ_N}+1)\cdot\sqrt{\log(2/\delta)+\log(\cN(\cF,1/N))}+1,
\]
which implies
\begin{align*}
\|\hat{f}-f^*\|_{\cZ_N}&\lesssim \sqrt{\log(2/\delta)+\log(\cN(\cF,1/N))}\\
&\lesssim \sqrt{\log(1/\delta)+\log N\cdot\dim_K(\cF)}\\
&\leq\sqrt{\log(1/\delta)+\log N\cdot\dim(\cF)}
\end{align*}
as desired.
\end{proof}

Following the analysis in \citet{russo2014learning}, we can bound the sum of bonuses in terms of the eluder dimension of $\cF$.
\begin{lem}
\[
\sum_{k=1}^Nb_k(z_k)\leq O(\dim_E(\cF)\log N+\sqrt{\dim_E(\cF)\cdot N\log N}\cdot \beta)
\]
\end{lem}
\begin{proof}
For $k\in[K]$, denote $\cZ_k=\{z_{\tau}\}_{\tau=1}^{k-1}$. For any given $\epsilon>0$ and $h\in[H]$, let $\mathcal{L}=\{z_k|k\in[N],b_k(z_k)>\epsilon\}$ with $|\mathcal{L}|=L$. We will show that there exists $z_k\in\mathcal{L}$ such that $z_k$ is $\epsilon$-dependent on at least $L/\dim_{E}(\mathcal{F},\epsilon)-1$ disjoint subsequences in $\cZ_k\cap \mathcal{L}$.
Denote $N=L/\dim_{E}(\mathcal{F},\epsilon)-1$. 

We decompose $\mathcal{L}$ into $N+1$ disjoint subsets, $\mathcal{L}=\cup_{j=1}^{N+1}\mathcal{L}_j$ by the following procedure. We initialize $\mathcal{L}_j=\{\}$ for all $j\in[N+1]$ and consider each $z_k\in \mathcal{L}$ sequentially. For each $z_k\in \mathcal{L}$, we find the smallest $1\leq j\leq N$ such that $z_k$ is $\epsilon$-independent on $\mathcal{L}_j$ with respect to $\mathcal{F}$. We set $j=N+1$ if such $j$ does not exist. We add $z_k$ into $\mathcal{L}_j$ afterwards. When the decomposition of $\mathcal{L}$ is finished, $\mathcal{L}_{N+1}$ must be nonempty since $\mathcal{L}_{j}$ contains at most $\dim_{E}(\mathcal{F},\epsilon)$ elements for $j\in[N]$. For any $z_k\in\mathcal{L}_{N+1}$, $z_k$ is $\epsilon$-dependent on at least $L/\dim_{E}(\mathcal{F},\epsilon)-1$ disjoint subsequences in $\cZ_k\cap \mathcal{L}$.

On the other hand, there exist $f_1,f_2\in\mathcal{F}$ such that $|f_1(z_k)-f_2(z_k)|>\epsilon$ and $  \|f_1-f_2\|^2_{\mathcal{Z}_k}\leq\beta^2$. By the definition of $\epsilon$-dependent we have
\[
(L/\dim_{E}(\mathcal{F},\epsilon)-1)\epsilon^2\leq\|f_1-f_2\|^2_{\mathcal{Z}_k}\leq\beta^2\]
which implies
\[
L\leq\left( \frac{\beta^2}{\epsilon^2}+1\right)\dim_{E}(\mathcal{F},\epsilon).
\]
Let $b_1\geq b_2\geq...\geq b_N$ be a permutation of $\{b_k(z_k)\}_{k\in[N]}$. For any $b_k\geq1/N$, we have 
\[
k\leq\left( \frac{\beta^2}{b_k^2}+1\right)\dim_{E}(\mathcal{F},b_k)\leq\left( \frac{\beta^2}{b_k^2}+1\right)\dim_{E}(\mathcal{F},1/N)
\]
which implies
\[
b_k\leq\left(\frac{k}{\dim_E(\mathcal{F},1/N)}-1 \right)^{-1/2}\cdot\beta .
\]
Moreover, we have $b_k\leq 1$. Therefore,
\begin{align*}
\sum_{k=1}^N b_k\leq& 1+\dim_E(\mathcal{F},1/N)+\sum_{\dim_E(\mathcal{F},1/N)<k\leq N}\left(\frac{k}{\dim_E(\mathcal{F},1/N)}-1 \right)^{-1/2}\cdot\beta\\
\leq& 1+\dim_E(\mathcal{F},1/N)+C\cdot\sqrt{\dim_E(\mathcal{F},1/N)\cdot N}\cdot \beta.\\
\leq& O(\dim_E(\cF)\log N+\sqrt{\dim_E(\cF)\cdot N\log N}\cdot \beta)
\end{align*}
as desired.
\end{proof}

\begin{proof}[Proof of Theorem~\ref{thm:arl}]
Note that the preference functions $\{b_k(\cdot)\}$ are non-increasing. Thus by Lemma 2, we have that
\[
\max_{z\in\cZ}b_N(z)\leq \frac1N \sum_{k=1}^Nb_k(z_k)\leq O
\left(
\frac{\dim_E(\cF)\log N+\sqrt{\dim_E(\cF)\cdot N\log N}\cdot \beta}{N}
\right)
\]
Substituting the value of $N$ and $\beta$, with a proper choice of $C_1$, we conclude that
\[
\max_{z\in\cZ}b_N(z)\leq \Delta/2
\]
Combining the above result with Lemma 1, we have for all $z\in\cZ$,
\[
|f^*(z)-\tilde{f}(z)|\leq \Delta/2
\]
which further implies \[
|f^*(z)-\hat{f}(z)|\leq \Delta
\]
Thus by the hard margin assumption of $f^*$, we complete the proof.
\end{proof}
\section{Proof of Theorem 2}
\label{proof2}
\subsection{Proof for the Linear Case}
We choose $\beta_{\text{lin}}$ to be:
\[
\beta_{\text{lin}}=C\cdot dH\sqrt{\dim(\cF)\log(dHK/\delta\Delta)}.
\]

Throughout the proof, we denote $\phi_h^k=\phi(s_h^k,a_h^k)$ for all $(h,k)\in[H]\times[K]$. 
We further denote 
\[
\Lambda_h^k=I+\sum_{\tau=1}^{k-1}\phi_h^\tau(\phi_h^\tau)^T
\]
We denote 
\[
w_h^k= \argmin_{w\in\RR^d} \sum_{\tau=1}^{k-1}(w^T \phi(s_h^\tau,a_h^\tau)-V_h^k(s_{h+1}^\tau))^2+\|w\|_2^2
\]
and thus $\widehat{P}_h^k V_{h+1}^k(\cdot,\cdot)=\phi(\cdot,\cdot)^T w_h^k$. Similarly,
\[
\overline{w}_h^k= \argmin_{w\in\RR^d} \sum_{\tau=1}^{k-1}(w^T \phi(s_h^\tau,a_h^\tau)-\overline{V}_h^k(s_{h+1}^\tau))^2+\|w\|_2^2
\]
and $\widehat{P}_h^k \overline{V}_{h+1}^k(\cdot,\cdot)=\phi(\cdot,\cdot)^T \overline{w}_h^k$.
We have the following lemma on the norm of $w_h^k$ and $\overline{w}_h^k$:
\begin{lem}
For all $(h,k)\in[H]\times[K]$,
\[\|w_h^k\|_2,\|\overline{w}_h^k\|_2\leq 2H\sqrt{dk}.\]
\end{lem}
\begin{proof}
The proof is identical to Lemma B.2 in \citet{jin2020provably}.
\end{proof}

\subsubsection{Analysis of the Planning Error}
Analysis in this section utilizes techniques from \citet{jin2020provably,wang2020reward}.

Denote all candidate reward function as $\cR$, which contains all functions in the from:
\[
r(\cdot)=\left\{\begin{aligned}
    1,\ \ f(\cdot)>1/2\\
    0,\ \ f(\cdot)\leq 1/2
\end{aligned}
\right.
\]
where $f\in\cC(\cF,\Delta/2)$. Clearly we have for all $h\in[H]$, the estimated reward function $\hat{r}_h\in\cR$.
Note that the size of $\cR$ is bounded by $\cN(\cF,\Delta/2)$.

Now we state the standard concentration bound for the linear MDP firstly introduced in \citet{jin2020provably}. 
\begin{lem}
\label{lem:confidence}
With probability at least $1-\delta$, for all $h\in[H]$ and $k\in[K]$,
\[
\left\|\sum_{\tau=1}^{k-1}\phi_h^\tau \left(V_{h+1}^k(s_{h+1}^\tau)-\sum_{s'\in\cS}P_h(s'|s_h^\tau,a_h^\tau)V_{h+1}^k(s')\right)\right\|_{(\Lambda_h^k)^{-1}}\leq C\cdot  dH\sqrt{\dim(\cF)\log(dKH/\delta\Delta)}
\]
\end{lem}
\begin{proof}
Note that the value function $V_{h+1}^k$ is of the form:
\begin{align}
\label{form}
V(\cdot)=\max_{a\in\cA}\Pi_{[0,H-h+1]}[r+w^T \phi(\cdot,a)+\min\{\beta_{\text{lin}}\cdot(\phi(\cdot,\cdot)^T(\Lambda)^{-1}\phi(\cdot,\cdot))^{1/2},H\}]
\end{align}
where $\|w\|_2\leq 2H\sqrt{dK}$, $\Lambda\succeq I$, and $r\in\cR$.

Consider a fixed $r\in\cR$. Identical to Lemma D.4 of~\citet{jin2020provably}, we have that: with probability at least $1-\delta$, for all $V(\cdot)$ in the above form~(\ref{form}) (with that fixed $r$),
\[
\left\|\sum_{\tau=1}^{k-1}\phi_h^\tau \left(V(s_{h+1}^\tau)-\sum_{s'\in\cS}P_h(s'|s_h^\tau,a_h^\tau)V(s')\right)\right\|_{(\Lambda_h^k)^{-1}}\leq C\cdot  dH\sqrt{\log(dKH/\delta)}.
\]
\end{proof}
By replacing $\delta$ with $\delta/|\cR|$ and applying a union bound over all $r\in\cR$, we have that with probability at least $1-\delta$,
\begin{align*}
\left\|\sum_{\tau=1}^{k-1}\phi_h^\tau \left(V(s_{h+1}^\tau)-\sum_{s'\in\cS}P_h(s'|s_h^\tau,a_h^\tau)V(s')\right)\right\|_{(\Lambda_h^k)^{-1}}&\lesssim   dH\sqrt{\log(dKH|\cR|/\delta)}\\
&\lesssim  dH\sqrt{\dim(\cF)\log(dKH/\delta\Delta)}.
\end{align*}
for all $V(\cdot)$ in the above form~(\ref{form}) (for all $r\in\cR$). And we are done.

The next lemma bound the single-step planning error. 
\begin{lem}[Single-Step Planning Error]
\label{planerror}
In Algorithm~\ref{alg:explore} and Algorithm~\ref{alg:plan}, 
with probability at least $1-\delta$,
for any $h\in[H]$, $k\in[K]$ and $(s,a)\in\cS\times\cA$,
\[
|\widehat{P}^k_h  {V}_{h+1}^k(s,a) - P_h {V}_{h+1}^k(s,a)|\leq \Gamma_h^k(s,a)
\]
and
\[
|\widehat{P}^k_h \overline{V}_{h+1}^k(s,a) - P_h \overline{V}_{h+1}^k(s,a)|\leq \Gamma_h^k(s,a).
\]
\end{lem}
\begin{proof}
We provide the proof for the first inequality and that for the second inequality is identical.

Note that
\begin{align*}
    P_h V_{h+1}^k (s,a)&= \sum_{s'\in\cS}P_h(s'|s,a)V_{h+1}^k(s')\\
    &=\phi(s,a)^T\left(\sum_{s'\in\cS}\mu_h(s')V_{h+1}^k(s')\right)
\end{align*}
We denote
\[
\tilde{w}_h^k=\sum_{s'\in\cS}\mu_h(s')V_{h+1}^k(s'),
\]
thus $P_h V_{h+1}^k (s,a) = \phi(s,a)^T \tilde{w}_h^k$. By $\|\mu_h(S)\|_2\leq \sqrt{d}$, we have $\|\tilde{w}_h^k\|_2\leq H\sqrt{d}$.

Note that
\begin{align*}
    &\phi(s,a)^T w_h^k- P_h {V}_{h+1}^k(s,a)\\
    &= \phi(s,a)^T (\Lambda_h^k)^{-1}\sum_{\tau=1}^{k-1}\phi_h^\tau\cdot V_{h+1}^k(s_{h+1}^\tau)-\phi(s,a)^T\tilde{w}_h^k\\
    &= \phi(s,a)^T (\Lambda_h^k)^{-1}\left(\sum_{\tau=1}^{k-1}\phi_h^\tau\cdot V_{h+1}^k(s_{h+1}^\tau)-\Lambda_h^k\tilde{w}_h^k\right)\\
    &= \phi(s,a)^T (\Lambda_h^k)^{-1}\left(\sum_{\tau=1}^{k-1}\phi_h^\tau \left(V_{h+1}^k(s_{h+1}^\tau)-\sum_{s'\in\cS}P_h(s'|s_h^\tau,a_h^\tau)V_{h+1}^k(s')\right)-\tilde{w}_h^k\right)
\end{align*}
Thus 
\begin{align*}
|\phi(s,a)^T w_h^k- P_h {V}_{h+1}^k(s,a)|\leq& \left|\phi(s,a)^T (\Lambda_h^k)^{-1}\sum_{\tau=1}^{k-1}\phi_h^\tau \left(V_{h+1}^k(s_{h+1}^\tau)-\sum_{s'\in\cS}P_h(s'|s_h^\tau,a_h^\tau)V_{h+1}^k(s')\right)\right|\\
&+|\phi(s,a)^T (\Lambda_h^k)^{-1}\tilde{w}_h^k|\\
\leq& \|\phi(s,a)\|_{(\Lambda_h^k)^{-1}}\cdot\left\|\sum_{\tau=1}^{k-1}\phi_h^\tau \left(V(s_{h+1}^\tau)-\sum_{s'\in\cS}P_h(s'|s_h^\tau,a_h^\tau)V(s')\right)\right\|_{(\Lambda_h^k)^{-1}}\\
&+\|\phi(s,a)\|_{(\Lambda_h^k)^{-1}}\cdot \|\tilde{w}_h^k\|_2\\
\leq& \Gamma_h^k(s,a)
\end{align*}
where the last inequality is obtained by plugging in the bound for $\|\tilde{w}_h^k\|_2$ and $$\left\|\sum_{\tau=1}^{k-1}\phi_h^\tau \left(V(s_{h+1}^\tau)-\sum_{s'\in\cS}P_h(s'|s_h^\tau,a_h^\tau)V(s')\right)\right\|_{(\Lambda_h^k)^{-1}}$$ (Lemma~\ref{lem:confidence}).
\end{proof}
The next lemma guarantees optimism in the planning phase.
\begin{lem}[Optimism]
For Algorithm~\ref{alg:plan}, with probability at least $1-\delta$,
for any $h\in[H+1]$, $k\in[K]$ and $(s,a)\in\cS\times\cA$,
\[
Q_h^k(s,a)\geq Q_h^*(s,a,\hat{r}),\ \ V_h^k(s)\geq V_h^*(s,\hat{r})
\]
\end{lem}
\begin{proof}
We condition on the event defined in Lemma~\ref{planerror}. The proof is by induction on $h$. The result for $h=H+1$ clearly holds. Suppose the result for $h+1$ holds. Note that for all $(s,a)\in\cS\times\cA$ and $k\in[K]$,
\begin{align*}
Q_h^k(s,a)&=\hat{r}_h(s,a)+\widehat{P}_h^k V_{h+1}^k(s,a) +\Gamma_h^k(s,a)\\
&\geq \hat{r}_h(s,a)+{P}_h^k V_{h+1}^k(s,a) \\
&\geq \hat{r}_h(s,a)+{P}_h^k V_{h+1}^*(s,a) \\
&= Q_h^*(s,a,\hat{r}).
\end{align*}
In the above proof we assume $Q_h^k(s,a)\leq H-h+1$, since we always have $Q_h^*(s,a,\hat{r})\leq H-h+1$. Moreover,
\[
V_h^k(s)=\max_{a\in\cA}Q_h^k(s,a)\geq \max_{a\in\cA}Q_h^*(s,a,\hat{r})= V_h^*(s,\hat{r})
\]
and we are done.
\end{proof}
The next lemma bound the regret in the planning phase in terms of the expected sum of exploration bonuses.
\begin{lem}[Regret Decomposition]
\label{lem7}
With probability at least $1-\delta$,
for any $h\in[H+1]$, $k\in[K]$, $s\in\cS$,
\[
V_h^k(s)-V_h^{\hat{\pi}^k}(s,\hat{r})\leq\overline{V}_h^k(s) 
\]
\end{lem}
\begin{proof}
We prove the lemma by induction on $h$. The conclusion clearly holds for $h=H+1$. Assume that the conclusion holds for $h+1$, i.e., for any $k\in[K]$ and $s\in\cS$,
\[
V_{h+1}^k(s)-V_{h+1}^{\hat{\pi}^k}(s,\hat{r})\leq\overline{V}_{h+1}^k(s) 
\]
Consider the case for $h$. Denote $a=\hat{\pi}_h^k(s)=\argmax_{a\in\cA}{Q}^k_h(\cdot,a)$ for the rest of the proof. We have
\[
V_h^k(s)=Q_h^k(s,a)=\Pi_{[0,H-h+1]}[\hat{r}_h(s,a) + \phi(s,a)^{T}{w}^k_h+\Gamma_h^k(s,a)]
\]
and
\[
V_{h}^{\hat{\pi}^k}(s,\hat{r})=Q_{h}^{\hat{\pi}^k}(s,a,\hat{r})=\hat{r}_h(s,a)+P_h V_{h+1}^{\hat{\pi}^k}(s,a,\hat{r})
\]
Thus we have
\begin{align*}
    V_h^k(s)-V_{h}^{\hat{\pi}^k}(s,\hat{r})&\leq \phi(s,a)^{T}{w}^k_h-P_h V_{h+1}^{\hat{\pi}^k}(s,a,\hat{r}) +\Gamma_h^k(s,a)\\
    &\leq P_h {V}_{h+1}^k(s,a)-P_h V_{h+1}^{\hat{\pi}^k}(s,a,\hat{r}) +2\Gamma_h^k(s,a)\\
    &\leq P_h \overline{V}_{h+1}^k(s,a) +2\Gamma_h^k(s,a)\\
    &\leq \phi(s,a)^T \overline{w}_h^k +3\Gamma_h^k(s,a)\\
    &\leq \overline{Q}_h^k(s,a)\\
    &\leq \overline{V}_h^k(s)
\end{align*}
as desired.
\end{proof}
\begin{lem}
\label{lem8}
With probability at least $1-\delta$, 
\[
\sum_{k=1}^K \overline{V}_1^k(s_1)\leq C\cdot \sqrt{\dim(\cF)d^3 H^4 K \log(dHK/\delta\Delta)}
\]
\end{lem}
\begin{proof}
Note that Algorithm~\ref{alg:explore} in the linear case is identical to the Algorithm 1 (LSVI-UCB) in \citet{jin2020provably} with zero reward, except for a enlarged bonus. $\sum_{k=1}^K \overline{V}_1^k(s_1)$ corresponds to the regret and can be estimated using standard techniques. We omit the proof for brevity.
\end{proof}

\begin{lem}
\label{9}
With probability at least $1-\delta$, 
\[
V_1^*(s_1,\hat{r})-V_h^{\hat{\pi}}(s_1,\hat{r})\leq C\cdot \sqrt{\frac{\dim(\cF)d^3 H^4 \log(dHK/\delta\Delta)}{K}}
\]
\end{lem}
\begin{proof}
We condition on the event defined in Lemma~\ref{lem7} and Lemma~\ref{lem8}. Note that 
\begin{align*}
V_1^*(s_1,\hat{r})-V_h^{\hat{\pi}}(s_1,\hat{r})&= \frac1K \sum_{k=1}^K\left( V_1^*(s_1,\hat{r})-V_h^{\hat{\pi}^k}(s_1,\hat{r})\right)\\
&\leq \frac1K \sum_{k=1}^K\left( V_1^k(s_1,\hat{r})-V_h^{\hat{\pi}^k}(s_1,\hat{r})\right)\\
&\leq \frac1K \left(\sum_{k=1}^K \overline{V}_1^k(s_1)\right)
\end{align*}
We complete the proof by plugging in the bound given in Lemma~\ref{lem8}.
\end{proof}
\subsubsection{Latent State Representation}
Our purpose is to show that $\hat{r}$ will not incur much error under \emph{any} policy. We need to exploit the latent  state structure of the MDP to bound the generalization error of $\hat{r}$. Firstly we need to derive the latent state model (a.k.a, soft state aggregation model) from the non-negative feature model.

Note that for all $(s,a)\in\cS\times\cA$,  $\sum_{s'\in\cS}P_h(s'|s,a)=1$, thus we have
\[
\left<\phi(s,a), \left(\sum_{s'\in\cS}\mu_h(s')\right)\right>=1
\]
Denote $\mu_h:=\sum_{s'\in\cS}\mu_h(s')$. We define a latent state space $\cX=\{1,2,...,d\}$.
Let each state-action pair induces a posterior distribution over $\cX$: 
\[
\psi_h:\cS\times\cA\rightarrow \Delta(\cX),\text{ where }\psi_h(s,a)[x]=\phi(s,a)[x] \cdot \mu_h[x].
\]
Since $\left<\phi(s,a),\mu_h\right>=1$, $\psi(\cdot)$ is a probability distribution.

For each latent variable induces a emission distribution over $\cS$
\[
\nu_{h}:\cX\rightarrow\Delta(S), \text{ where } \nu_h(x)[s']=\mu_h(s')[x]/\mu_h[x].
\]
$\nu_h(\cdot)$ is also a probability distribution by definition.
In stage $h$ we sample $x_h\sim\psi(s_h,a_h)$ and $s_{h+1}\sim\nu_h(x_h)$. The trajectory can be amplified as:
\[
s_1,a_1,x_1,s_2,...,s_H,a_H,x_H,s_{H+1}.
\]
It suffice to check the transition probability is maintained:
\begin{align*}
\mathbb{P}(s_{h+1}|s_h,a_h) &=\sum_{x=1}^d \psi_h(s_h,a_h)[x]\cdot \nu_h(x)[s_{h+1}] \\
&= \sum_{x=1}^d (\phi(s_h,a_h)[x]/\mu_h[x])\cdot (\mu_h(s')[x]/\mu_h[x]) \\
&= \left<\phi(s_h,a_h),\mu_h(s')\right>\\
&=P_h(s'|s,a)
\end{align*}
and we are done.
\subsubsection{Analysis of the Reward Error}
The error can be decomposed in the following manner.
\begin{lem}
For any policy $\pi$, we have that
\[
\Big|V_1^{\pi}(s_1,r)-V_1^{\pi}(s_1,\hat{r})\Big|\leq \sum_{h=1}^{H}\mathbb{E}_{\pi}\left[\sum_{a\in\cA}\left|r(s_{h+1},a)-\hat{r}(s_{h+1},a)\right|\right]
\]
\end{lem}

From now on we fix a stage $h\in[H-1]$, and try to analyze the error of the learned reward function in stage $h+1$. We leverage the latent variable structure to analyze the error of $\hat{r}$. 
For $j\in[d]$, denote $c_j$ the number of times we visit the $j$-th latent state in stage $h$ during $K$ episodes.
\[
c_j=\sum_{i=1}^K \one\{x_h^i=j\}.
\]

We define the error of $\hat{r}$ starting from the $j$-th latent state as:
\begin{align*}
w[j]&=\EE\left[|r_{h+1}(s_{h+1},a)-\hat{r}_{h+1}(s_{h+1},a)|\Big|s_{h+1}\sim \mu_h(j),a\sim\unif(\cA)\right]\\
&=\frac{1}{|\cA|}\sum_{s'\in\cS}\nu_{h}(s')[j]\sum_{a\in\cA}\left|r_{h+1}(s',a)-\hat{r}_{h+1}(s',a)\right|
\end{align*}
where $\nu_h$ denotes the emission probability. 
Thus we can further define the error vector of $\hat{r}$ as
\[
w=\frac{1}{|\cA|}\sum_{s'\in\cS}\nu_h(s')\sum_{a\in\cA}\left|r_{h+1}(s',a)-\hat{r}_{h+1}(s',a)\right|
\]
Denoting $\phi_{\pi}=\mathbb{E}_{\pi}\phi(s_h,a_h)$. The next key lemma bound the error induced by the reward function. The proof of Lemma~\ref{lem:bound_error} is defered to the next section.
\begin{lem}
\label{lem:bound_error}
With probability at least $1-\delta$, for any policy $\pi$,
\[
\mathbb{E}_{\pi}\left[\sum_{a\in\cA}\left|r_{h+1}(s_{h+1},a)-\hat{r}_{h+1}(s_{h+1},a)\right|\right]\leq C\cdot\|\phi_{\pi}\|_{(\Lambda_h^K)^{-1}}\cdot|\cA|\cdot \sqrt{d\sum_{j=1}^d(c_j w_j)^2+d^2\log^2(K/\delta)} 
\]
for some absolute constant $C>0$.
\end{lem}

\begin{lem}
With probability at least $1-\delta$,
\[
\sum_{h=1}^H\|\phi_{\pi}\|_{(\Lambda_h^K)^{-1}}\lesssim \sqrt{\frac{\dim(\cF)d^3H^4\cdot\log(dHK/\delta\Delta)}{K}}
\]
\end{lem}
\begin{proof}
Similar to Lemma 3.2 of \citet{wang2020reward},
with probability at least $1-\delta$, for any policy $\pi$, we have that (we treat $\Gamma_h^K(\cdot,\cdot)/H$ as a reward function)
\[
V_1^{\pi}(s_1,\Gamma^K/H)\lesssim\sqrt{\frac{\dim(\cF)d^3H^4\cdot\log(dHK/\delta\Delta)}{K}}, 
\]
where $\Gamma^K_h(s_h,a_h)=\min \{\beta_{\text{lin}}\sqrt{\phi(s_h,a_h)^{T}(\Lambda_h^K)^{-1}\phi(s_h,a_h)},H\}$.
We condition on this event for the rest of the proof. Note that
\[
\Gamma_h^K\geq H\cdot \sqrt{\phi(s_h,a_h)^{T}(\Lambda^K_h)^{-1}\phi(s_h,a_h)}.
\]
Thus we have
\[
\mathbb{E}_{\pi}\sqrt{\phi(s_h,a_h)^{T}(\Lambda^K_h)^{-1}\phi(s_h,a_h)}\lesssim\sqrt{\frac{\dim(\cF)d^3H^4\cdot\log(dHK/\delta\Delta)}{K}}
\]
Note that by Jensen's inequality,
\begin{align*}
\mathbb{E}_{\pi}\sqrt{\phi(s_h,a_h)^{T}\Lambda_h^{-1}\phi(s_h,a_h)}&=\mathbb{E}_{\pi}\|\phi(s_h,a_h)\|_{\Lambda_h^{-1}}\\
&\geq \|\mathbb{E}_{\pi}\phi(s_h,a_h)\|_{\Lambda_h^{-1}}
\end{align*}
and we are done.
\end{proof}
For a distribution $\lambda\in\Delta(\cS\times\cA)$, denote the population risk of an estimated reward function $\hat{r}$ in the $(h+1)$-th stage as
\[
\err_{\lambda}(\hat{r})=P_{\lambda}(\{r_{h+1}(s_{h+1},a_{h+1})\ne \hat{r}_{h+1}(s_{h+1},a_{h+1})\}).
\]
For $j\in[d]$, denote $\lambda_j$ the distribution starting from the $j$-th hidden state and take random action, i.e.,
\[
\lambda_{j}=\nu_h(j)\times\unif(\cA).
\]
Then the error vector can be represented as
\[
w[j]=\err_{\lambda_j}(\hat{r}).
\]
Note that every time we arrive at $j$-th hidden state, i.e., $x_h^k=j$, it indicates that $(s_{h+1}^k, \tilde{a}^{k}_{h+1})$ is a random sample from $\lambda_j$. Denote the empirical risk of $\hat{r}$ for the first $m$ samples from $\lambda_j$ as $\err_{\nu_j,m}(\hat{r})$. Classic supervised learning theory gives us the following bound.
\begin{lem}
With probability at least $1-\delta$, for all $m\in[K]$ and reward function $r\in\cR$ consistent with the first $m$ samples from $\nu_j$,
\begin{align*}
\err_{\lambda_j}(\tilde{r})&\leq \frac1m(\log|\cR|+\log(K/\delta))\\
&\lesssim \frac1m(\dim(\cF)\log(1/\Delta)+\log(K/\delta))\\
&\lesssim \frac1m(\dim(\cF)\cdot\log(K/\delta\Delta))
\end{align*}
\end{lem}

We conclude that with probability at least $1-\delta$, for any policy $\pi$,
\begin{align*}
\sum_{h=1}^H\mathbb{E}_{\pi}\left[\sum_{a\in\cA}\left|r(s_{h+1},a)-\hat{r}(s_{h+1},a)\right|\right]&\lesssim|\cA|\cdot\sqrt{\frac{\dim(\cF)d^3H^4\cdot\log(dHK/\delta\Delta)}{K}}\cdot \sqrt{d\sum_{j=1}^d(c_j w_j)^2+d^2\log^2(K/\delta)}\\
&\lesssim |\cA|\cdot\sqrt{\frac{\dim(\cF)d^3H^4\cdot\log(dHK/\delta\Delta)}{K}}\cdot \sqrt{d^2\dim^2(\cF)\log^2(K/\Delta\delta)}\\
&\lesssim |\cA|\cdot \sqrt{\frac{d^5\dim^3(\cF)H^4\log^3(dHK/\delta\Delta)}{K}}
\end{align*}
By the above results we conclude the following lemma.
\begin{lem}
\label{14}
With probability at least $1-\delta$,
\[
\sup_{\pi}\Big|V_1^{\pi}(s_1,r)-V_1^{\pi}(s_1,\hat{r})\Big|\leq |\cA|\cdot \sqrt{\frac{d^5\dim^3(\cF)H^4\log^3(dHK/\delta\Delta)}{K}},
\]
\end{lem}
\begin{proof}[Proof of Theorem 2 in the Linear Case] Note that
\[
V_1^{\pi^*}-V_1^{\hat{\pi}}\leq |V_1^{{\pi}^*(\hat{r})}(\hat{r})-V_1^{\hat{\pi}}(\hat{r})|+|V_1^{\pi^*}-V_1^{\pi^*}(\hat{r})| + |V_1^{\hat{\pi}}-V_1^{\hat{\pi}}(\hat{r})|.
\]
Combing Lemma~\ref{9} and Lemma~\ref{14} completes the proof.
\end{proof}
\subsubsection{Proof of Lemma~\ref{lem:bound_error}}
Denote that \[\Lambda_h=\sum_{k=1}^K \phi(s_h^k,a_h^k)\phi(s_h^k,a_h^k)^T+I.
\]
We define an expected version of $c_j$:
\[
e_j=\sum_{i=1}^K \psi(s_h^i,a_h^i)[j]
\]
The next lemma bound $e_j$ in terms of $c_j$.
\begin{lem}
With probability at least $1-\delta$, for all $j\in[d]$,
\[
e_j\leq C\cdot\max\left\{c_j,\log(K/\delta)\right\}
\]
for some absolute constant $C>0$.
\end{lem}
\begin{proof}[Proof of Lemma~\ref{lem:bound_error}]

Note that 
\begin{align*}
P_{\pi}[s_{h+1}=s']&=\sum_{s,a}P_{\pi}[s_{h}=s,a_{h}=a]\cdot \phi(s,a)^{T}\mu(s')\\
&=\mathbb{E}_{\pi}[\phi(s_h,a_h)]^T\mu(s')\\
&=(\phi_{\pi})^T\mu(s')
\end{align*}
Thus the error caused by $\hat{r}$ in stage $h$ can be represented as:
\begin{align*}
\mathbb{E}_{\pi}\sum_{a\in\cA}\left|r(s_{h+1},a)-\hat{r}(s_{h+1},a)\right|
&=\sum_{s'\in\cS}\left(P_{\pi}[s_{h+1}=s']\sum_{a\in\cA}\left|r(s',a)-\hat{r}(s',a)\right|\right)\\
&=(\phi_{\pi})^T \sum_{s'\in\cS}\mu(s')\sum_{a\in\cA}\left|r(s',a)-\hat{r}(s',a)\right|\\
&=(\phi_{\pi})^T \cdot |\cA|w'
\end{align*}
where $w'[j]=w[j]\cdot \mu_h[j]$

Here we bound the error vector $w'$ under the $\Lambda_h$-norm. 
For $i\in[K]$, denote $\phi_i=\phi(s_h^i,a_h^i)$. Then we have that
\begin{align*}
\|w'\|_{\Lambda_h}^2&=(w')^T (\sum_{i=1}^K\phi_i\phi_i^{T}+I_{d})(w')\\
&\leq\sum_{i=1}^K(\phi_i^T w')^2+d\\
&\leq d\sum_{i=1}^K\sum_{j=1}^d(\phi_{i}[j] w'[j])^2+d\quad (\text{Cauchy-Schwartz inequality})\\
&\leq d\sum_{i=1}^K\sum_{j=1}^d(\psi_{i}[j] w[j])^2+d\\
&\leq d\sum_{j=1}^d\left(\sum_{i=1}^K\psi_{i}[j] w[j]\right)^2+ d \quad (\text{Note that }\phi_{i}[j] w[j]\geq 0)\\
&= d\sum_{j=1}^d(e_j w[j])^2+ d\\
&\lesssim d\sum_{j=1}^d(c_j w[j]+ \log(K/\delta))^2+ d\\
&\lesssim d\sum_{j=1}^d(c_j w[j])^2+d^2\log^2(K/\delta)
\end{align*}
Thus we have that
\begin{align*}
(\phi_{\pi})^T w' &\leq \|\phi_{\pi}\|_{\Lambda_h^{-1}}\cdot \|w'\|_{\Lambda_h}\\
&\lesssim \|\phi_{\pi}\|_{\Lambda_h^{-1}}\cdot \sqrt{d\sum_{j=1}^d(c_j w[j])^2+d^2\log^2(K/\delta)} 
\end{align*}
\end{proof}
\subsection{Proof of the Tabular Case}
We choose $\beta_{\text{tbl}}$ to be:
\[
\beta_{\text{tbl}}=C\cdot H\sqrt{\log(SAHK/\delta)}.
\]
Before the proof, we remark that directly treating the linear case as a special case of the tabular case will derive a much looser bound. Our proof is based on the analysis in \citet{wu2021accommodating} and \citet{zanette2019tighter}. We streamline the key lemmas and omit some of the detailed proofs for brevity.
\subsubsection{Good Events}
Denoting $w_h^k(s,a)=\mathbb{P}_{\pi^k}\{(s_h,a_h)=(s,a)\}$, we construct the following ``good event''.
\[
G_H=
\left\{
\forall (s,a,h,k),|(\widehat{P}_h^k-P_h)V_{h+1}^*(s,a)|\leq H\sqrt{\frac{\log(SAHK/\delta)}{N_h^k(s,a)}}
\right\}
\]

\[
G_P=
\left\{
\forall (s,a,s',h,k),|(\widehat{P}_h^k-P_h)(s'|s,a)|\leq 2\sqrt{\frac{P_h(s'|s,a)\log(SAKH/\delta)}{N_h^k(s,a)}}+\frac{4\log(SAKH/\delta)}{N_h^k(s,a)}
\right\}
\]

\[
G_{\widehat{P}}=
\left\{
\forall (s,a,s',h,k),|(\widehat{P}_h^k-P_h)(s'|s,a)|\leq 2\sqrt{\frac{\widehat{P}^k_h(s'|s,a)\log(SAKH/\delta)}{N_h^k(s,a)}}+\frac{4\log(SAKH/\delta)}{N_h^k(s,a)}
\right\}
\]
\[
G_N=
\left\{
\forall (s,a,h,k), N_h^k(s,a)\geq \frac12 \sum_{\tau=1}^{k-1} w_h^\tau(s,a)-\log(SAKH/\delta)
\right\}
\]
\begin{lem}
\[
\mathbb{P}\{G_{H}\cap G_{P}\cap G_{\widehat{P}}\cap G_N\}\geq 1-4\delta
\]
\end{lem}
\begin{proof}
The proof is identical to that of Lemma 1 in \citet{wu2021accommodating}. We omit it for brevity.
\end{proof}
\begin{lem}
\label{18}
If events $G_P$, $G_{\widehat{P}}$ hold, then for all $V_1,V_2:\cS\rightarrow[0,H]$ satisfying $V_1\leq V_2$ and $(s,a)\in\cS\times\cA$
\[
\Big|(\widehat{P}_h^k-P_h)(V_2-V_1)(s,a)\Big|\leq \frac1H P_h(V_2-V_1)(s,a)+\frac{5H^2S\log(SAHK/\delta)}{N_h^k(s,a)}
\]
and
\[
\Big|(\widehat{P}_h^k-P_h)(V_2-V_1)(s,a)\Big|\leq \frac1H \widehat{P}_h^k(V_2-V_1)(s,a)+\frac{5H^2S\log(SAHK/\delta)}{N_h^k(s,a)}.
\]
\end{lem}
\begin{proof}
The proof is identical to that of Lemma 3 in \citet{wu2021accommodating}. We omit it for brevity.
\end{proof}
\subsubsection{Analysis}
\begin{lem}[Optimism of the planning phase]
If $G_H$ holds, for all $(s,a)\in\cS\times\cA$, $h\in[H]$ and $k\in[K]$, 
\[
V_h^*(s)\leq V_h^k(s),\ Q_h^*(s,a)\leq Q_h^k(s,a)
\]
\end{lem}
\begin{proof}
Note that the estimated reward function $\hat{r}_h$ is always true for $N_h^k(s,a)>0$. On the other hand, for $N_h^k(s,a)=0$, the optimistic Q-function $Q_h^k(s,a)$ is $H-h+1$ and the value of $\hat{r}_h$ will not affect $Q_h^k(s,a)$. The rest of the proof follows from standard techniques from \citet{azar2017minimax}.  
\end{proof}
\begin{lem}
\label{19}
If events $G_H$, $G_{\widehat{P}}$ holds, then for all $s\in\cS$, $h\in[H]$ and $k\in[K]$,
\[
V_h^k(s)-V_h^{\hat{\pi}^k}\leq \left(1+\frac1H\right)^{H-h+1}\cdot \overline{V}^k_h(s)
\]
In particular,
\[
V_h^k(s)-V_h^{\hat{\pi}^k}(s)\leq e\cdot \overline{V}^k_h(s)
\]
\begin{proof}
The proof is identical to that of Lemma 10 in \citet{wu2021accommodating}. We omit it for brevity.
\end{proof}
\end{lem}
\begin{lem}
If events $G_H$, $G_P$ holds, then for all $k\in[K]$,
\[
\overline{V}_1^k(s_1) \leq \EE_{s_h,a_h\sim\pi^k}\sum_{h=1}^H H\wedge \left( \sqrt{\frac{H^2 \iota}{N_h^k(s_h,a_h)}}+\frac{H^2S\iota}{N_h^k(s_h,a_h)} \right)
\]
\end{lem}
\begin{proof}
We denote $a_1=\pi_h^k(s_1)$. Note that
\begin{align*}
    \overline{V}_1^k(s_1)&=\overline{Q}_1^k(s_1,a_1)\\
    &\leq \widehat{P}_1^k \overline{V}_2^k(s_1,a_1) +b_1^k(s_1,a_1)\\
    &= P_1\overline{V}_2^k(s_1,a_1)+(\widehat{P}_1^k-P_1) \overline{V}_2^k(s_1,a_1) +b_1^k(s_1,a_1)\\
    &= (1+\frac1H)P_1\overline{V}_2^k(s_1,a_1)+\frac{5H^2S\iota}{N_1^k(s_1,a_1)} +b_1^k(s_1,a_1)\\
    &= (1+\frac1H)P_1\overline{V}_2^k(s_1,a_1)+\frac{5H^2S\iota}{N_1^k(s_1,a_1)} +b_h^k(s_1,a_1)\\
    &\leq...\\
    &\leq (1+\frac1H)^H\EE_{\pi^k} \sum_{h=1}^H\left(\frac{5H^2S\iota}{N_h^k(s_1,a_1)} +b_h^k(s_1,a_1)\right)\\
\end{align*}
and we are done.
\end{proof}
\begin{lem}
\label{21}
With probability at least $1-\delta$,
\[
\sum_{k=1}^K \overline{V}_1^k(s_1)\lesssim \sqrt{H^4SAK\iota}+H^3S^2A\iota,\text{ where } \iota=\log(HSAK/\delta).
\]
\end{lem}
\begin{proof}
We set $L_h^k=\{(s,a)|\sum_{\tau=1}^{k-1}w_h^\tau(s,a)\geq 2\iota\}$.
Note that
\begin{align*}
    \EE_{\pi^k}\sum_{k=1}^K\sum_{h=1}^H H\wedge \left( \sqrt{\frac{H^2 \iota}{N_h^k(s_h,a_h)}}+\frac{H^2S\iota}{N_h^k(s_h,a_h)} \right)\\
    =\sum_{k=1}^K\sum_{h=1}^H \sum_{s,a}w_h^k(s,a) H\wedge \left( \sqrt{\frac{H^2 \iota}{N_h^k(s_h,a_h)}}+\frac{H^2S\iota}{N_h^k(s_h,a_h)} \right).
\end{align*}
We estimate these parts separately. 
By definition we have
\[
\sum_{k=1}^K\sum_{h=1}^H \sum_{(s,a)\notin L_h^k}w_h^k(s,a) H\leq 2H^2SA\iota
\].
Note that
\begin{align*}
\sum_{k=1}^K\sum_{h=1}^H \sum_{(s,a)\in L_h^k}w_h^k(s,a) \sqrt{\frac{H^2 \iota}{N_h^k(s_h,a_h)}}
&\lesssim \sum_{k=1}^K\sum_{h=1}^H \sum_{(s,a)\in L_h^k}w_h^k(s,a) \sqrt{\frac{H^2 \iota}{\sum_{\tau=1}^{k-1}w_h^{\tau}(s,a)}}\\
&\lesssim \sqrt{H^2\iota} \cdot HSA \cdot \sqrt{K}\\
&=\sqrt{H^4SAK\iota}
\end{align*}
and
\begin{align*}
\sum_{k=1}^K\sum_{h=1}^H \sum_{(s,a)\in L_h^k}w_h^k(s,a) \frac{H^2S \iota}{N_h^k(s_h,a_h)}
&\lesssim \sum_{k=1}^K\sum_{h=1}^H \sum_{(s,a)\in L_h^k}w_h^k(s,a) \cdot\frac{H^2S \iota}{\sum_{\tau=1}^{k-1}w_h^{\tau}(s,a)}\\
&\lesssim {H^2S\iota} \cdot HSA \cdot \iota\\
&=H^3S^2A\iota^2.
\end{align*}
Combining the above three parts we complete the proof.
\end{proof}
\begin{lem}
\label{22}
\[
V_1^*(s_1)-V_1^{\hat{\pi}}(s_1)\lesssim \sqrt{\frac{H^4SA\iota}{K}}+\frac{H^3S^2A\iota^2}{K},\text{ where } \iota=\log(HSAK/\delta).
\]
\end{lem}
\begin{proof}
Combining the results in Lemma~\ref{18},Lemma~\ref{19} and Lemma~\ref{21} completes the proof.
\end{proof}
\begin{proof}[Proof of Theorem 2 in the Tabular Case] Plugging in the value of $K$ into Lemma~\ref{22} completes the proof.
\end{proof}
\section{Proof of Theorem 3}
\label{proof3}
We choose $\beta_{\text{tbl}}'$ to be:
\[
\beta_{\text{tbl}}'=C\cdot H\sqrt{S\log(SAHK/\delta)}.
\]
By standard techniques developed in \citet{jaksch2010near}, we bound the L1-norm of the estimation error of $\widehat{P}_h$ in the following sense.
\begin{lem}
\label{23}
For $\tau\in[K]$, $h\in[H]$ and $(s,a)\in\cS\times\cA$, denote $\widehat{{P}}_h^\tau(\cdot|s,a)\in\RR^{S}$ the empirical estimation of $P_h(\cdot|s,a)$ based on the first $\tau$ samples from $(s,a)$ in $\cD$. Then with probability at least $1-\delta$,
\[
\|\widehat{{P}}_h^\tau(\cdot|s,a)-P_h(\cdot|s,a)\|_1\leq C\cdot\sqrt{\frac{S\log(SAK/\delta)}{\tau}}
\]
\end{lem}
The next lemma bound the single-step planning error.
\begin{lem}
\label{24}
With probability at least $1-\delta$, for all $(s,a)\in\cS\times\cA$ and all $h\in[H]$,
\[
|\widehat{{P}}_h \widehat{V}_{h+1}(s,a)-{P}_h \widehat{V}_{h+1}(s,a)|\leq \Gamma_h(s,a),
\]
and
\[
|\hat{r}_h(s,a)-r_h(s,a)|\leq \Gamma_h(s,a).
\]
\end{lem}
\begin{proof}
Note that 
\[|\widehat{{P}}_h \widehat{V}_{h+1}(s,a)-{P}_h \widehat{V}_{h+1}(s,a)|\leq \|\widehat{{P}}_h^\tau(\cdot|s,a)-P_h(\cdot|s,a)\|_1\cdot \|\widehat{V}_{h+1}\|_{\infty}.\]
The proof of the first part follows from the results in Lemma~\ref{23}. The second part is obvious.
\end{proof}
Define the model evaluation error to be
\[
\iota_h(s,a)= (\mathbb{P}_h \widehat{V}_{h+1})(s,a)+r_h(s,a)-\widehat{Q}_h(s,a).
\]
\begin{lem}
\label{25}
Under the event defined in Lemma~\ref{24}, for all $(s,a)\in\cS\times\cA$ and all $h\in[H]$,
\begin{align*}
0\leq \iota_h(s,a) &\leq 4\Gamma_h(s,a)
\end{align*}
\end{lem}
\begin{proof}[Proof of Theorem~\ref{thm:pess}]
With Lemma~\ref{25}, Theorem 3 falls into a special case of Theorem 4.2 of \citet{jin2021pessimism}. We omit the complete proof for brevity.
\end{proof}

\section{Extensions}
\subsection{Unknown Noise Margin}
\label{sec:unknown}
In the main paper we assume that the noise margin $\Delta$ is known as a prior, and the algorithms need $\Delta$ as a input. But in reality the value of $\Delta$ is usually unknown to the agent. Here we provide an approach to bypass this issue. We use binary search to guess the value of $\Delta$. This only introduces a log factor to the asymptotic sample complexity as we only need to guess logarithmically many times.
\begin{algorithm}[h]
	\caption{Active Reward Learning with Validation ($\cZ$, $\Delta$, $\delta$)\label{alg:active_reward_learning_validation}}
	\begin{algorithmic}
		\STATE \textbf{Input:} Data Pool $\cZ=\{z_i\}_{i\in[T]}$, guess margin $\Delta$, failure probability $\delta\in(0,1)$
		\STATE $\cZ_0\leftarrow \{\}$ //Query Dataset
		\STATE Set
		$
		N\leftarrow C_1\cdot \frac{(\dim^2(\cF)+\dim(\cF)\cdot \log(1/\delta))\cdot (\log^2(\dim(\cF)))}{\Delta^2}
		$
	    \FOR{$k=1,2,...,N$}
		\STATE $\beta\leftarrow C_2\cdot \sqrt{\log(1/\delta)+\log N\cdot\dim(\cF)}$
		\STATE Set the bonus function: 
		$	b_k(\cdot)\leftarrow \sup_{f,f'\in\cF,\|f-f'\|_{\cZ_{k-1}}\leq \beta}|f(\cdot)-f'(\cdot)|
		$
		
		\STATE $z_k\leftarrow \arg\max_{z\in\cZ}b_k(z)$
		\STATE  $\cZ_{k}\leftarrow\cZ_{k-1}\cup\{z_k\}$
		\ENDFOR
		\FOR{$z\in\cZ_N$}
		\STATE Ask the human expert for a label $l(z)\in\{0,1\}$
		\ENDFOR
		\STATE Estimate the human model as
		$	\tilde{f}=\arg\min_{f\in\cF}\sum_{z\in\cZ_N}(f(z)-l(z))^2
		$
		{
		\FORALL{$z\in\cZ$}
		\IF{$|\tilde{f}(z)-1/2|>\Delta/2$}
		\RETURN \textbf{false}
		\ENDIF
		\ENDFOR}
		\STATE Let $\hat{f}\in\cC(\cF,\Delta/2)$ such that $\|\hat{f}-\tilde{f}\|_{\infty}\leq \Delta/2$
		\STATE Estimate the underlying true reward:
		$
		\hat{r}(\cdot)=\left\{\begin{aligned}
    1,\ \ \hat{f}(\cdot)>1/2\\
    0,\ \ \hat{f}(\cdot)\leq 1/2
\end{aligned}
\right.
		$
		\STATE \textbf{return:} The estimated reward function $\hat{r}$.
	\end{algorithmic}
\end{algorithm}
\begin{algorithm}[h]
	\caption{GuessDelta\label{alg:guess}}
	\begin{algorithmic}
		\FOR{$n=1,2,\ldots$}
		\STATE $\Delta'\leftarrow\frac{1}{2^n}$
		\STATE Run Algorithm~\ref{alg:explore} and Algorithm~\ref{alg:plan} (equipped with Algorithm~\ref{alg:active_reward_learning_validation}) with guess margin $\Delta'$ and confidence parameter $\frac{\delta}{n(n+1)}$
		\IF{A policy $\pi$ is returned from Algorithm~\ref{alg:plan}}
		\RETURN $\pi$
		\ENDIF
		\ENDFOR
	\end{algorithmic}
\end{algorithm}

First, we add a validation step in the active learning algorithm. After learning the human model $\tilde{f}$, we test whether for each data point $z$ in the data pool $\cZ$ we have $\tilde{f}(z)>\Delta/2$. If this is true, the reward labels of the data points in the data pool is guaranteed to be right, which is enough to guarantee the accuracy of the learned reward function. Otherwise we halt the algorithm and try the next guess of $\Delta$. The full algorithm is presented in Algorithm~\ref{alg:active_reward_learning_validation}. 

Now we introduce the procedure for guessing $\Delta$. We set $\Delta’=1/(2^n), (n=1,2,...)$ and run Algorithm 2 and Algorithm 3 repeatedly. For a guess of $\Delta$, the output policy is guaranteed to be near-optimal if the algorithms successfully finish and have not been halted by the validation step. Otherwise, we replace $\Delta’$ with $\Delta’/2$ and rerun the whole algorithm. The doubling schedule implies that the smallest guess is at least $\Delta/2$. Besides, we also need to adjust the confidence parameter to $\delta/(n(n+1))$. The whole procedure for guessing $\Delta$ is presented in Algorithm~\ref{alg:guess}.

We state the theoretical guarantee in Theorem~\ref{thm:main_guess}.
\begin{thm}
\label{thm:main_guess}
In the linear case, Algorithm~\ref{alg:guess} can find an $\epsilon$-optimal policy with probability at least $1-\delta$, with at most
\[
{O}\left(\frac{|\cA|^2d^5\dim^3(\cF)H^4\iota^4}{\epsilon^2}\right),\ \ \iota=\log(\frac{HSA}{\epsilon\delta\Delta})
\]
episodes. In the tabular case, Algorithm~\ref{alg:guess} can find an $\epsilon$-optimal policy with probability at least $1-\delta$, with at most
\[
O\left(\frac{H^4SA\iota^2}{\epsilon^2}+\frac{H^3S^2A\iota^3}{\epsilon}\right), \ \ \iota=\log(\frac{HSA}{\epsilon\delta\Delta})
\]
episodes. In both cases, the total number of queries to the reward is bounded by $\tilde{O}(H\cdot\dim^2(\cF)/\Delta^2)$.
\end{thm}
\begin{proof}
Note that
\[
\sum_{n=1}^\infty \frac{\delta}{n(n+1)}=\delta.
\]
Thus we can condition on the good events defined in the proof of Theorem~\ref{thm:arl} and Theorem~\ref{thm:main} for all $n\in\mathbb{N}$. Note that
\begin{itemize}
    \item With the validation step, the output policy of Algorithm~\ref{alg:plan} is guaranteed to be $\epsilon$-optimal, regardless of whether the guess $\Delta'$ is true. 
    \item Algorithm~\ref{alg:plan} will output a policy whenever $\Delta'<\Delta$.
\end{itemize}
As a result, Algorithm~\ref{alg:guess} will terminate with an $\epsilon$-optimal policy, and with at most $O(\log(1/\Delta))$ guesses of $\Delta$. Note that the sample and feedback complexity bounds are both monotonically increasing in $1/\Delta$, thus they at most multiply a log facter $\log(1/\Delta)$. The difference in the confidence parameter won't effect the bound.
\end{proof}
\subsection{Beyond Binary Reward}
\label{sec:beyondbinary}
In the main paper we assume that the valid reward function is binary. We remark that our framework can be generalized to RL problems with $n$-uniform discrete rewards. We consider a fixed stage $h\in[H]$, and omit the subscript $h$ in this section.

In this case, the reward function takes value from $\{0,\frac1n,\frac2n,\ldots, 1\}$. In each query, the human teacher chooses from $\{0,\frac1n,\frac2n,\ldots, 1\}$ (when $n=2$, the choices are $\{0, \frac12, 1\}$, which can be interpreted as “bad”, “average”, and “good” actions). We assume that when queried about a data point $z=(s,a)$, the probability of the human teacher choosing $\frac{i}{n}$ is $p_i(z)$ ($0\leq i\leq n$), and the human response model $f^*$ satisfies:
\[
\sum_{i=0}^n p_i(z) \cdot \frac{i}{n} = f^*(z)
\]
where $f^*$ belongs to the pre-specified function class $\cF$. We assume the true reward of $z$ is determined by $f^*(z)$. Concretely, 
\[
r(z)=\left\{\begin{aligned}
    &1,& &f^*(z)\in(\frac{2n-1}{2n},1],\\
    &\frac{i}{n},& &f^*(z)\in (\frac{2i-1}{2n},\frac{2i+1}{2n}], \ \ (1\leq i \leq n-1)\\
    &0,& &f^*(z)\in[0,\frac{1}{2n}].
\end{aligned}
\right.
\]
The bounded noise assumption becomes that $f^*(z)$ can not be too near the decision boundary.
\begin{assum}[Bounded Noise in Uniform Discrete Rewards Setting]
\label{assum:margin_multi}
There exists $\Delta>0$, such that for all $z\in\cS\times\cA$, and all $1\leq i\leq n$,
\[
|f^*(z)-\frac{2i-1}{2n}|>\Delta.
\]
\end{assum}
In Algorithm~\ref{alg:active_reward_learning} we estimate the underlying true reward as
\[
\hat{r}(z)=\left\{\begin{aligned}
    &1,& &\hat{f}(z)\in(\frac{2n-1}{2n},1],\\
    &\frac{i}{n},& &\hat{f}(z)\in (\frac{2i-1}{2n},\frac{2i+1}{2n}], \ \ (1\leq i \leq n-1)\\
    &0,& &\hat{f}(z)\in[0,\frac{1}{2n}].
\end{aligned}
\right.
\]
The other parts of the algorithm are similar to that with binary rewards. Following similar analysis in the proof of Theorem~\ref{thm:arl}, we can learn the reward labels in the data pool correctly using only $\widetilde{O}(\frac{d^2}{\Delta^2})$ queries. Thus we can derive the exact same sample and feedback complexity bounds as in the binary reward case.

\subsection{Beyond Bounded Noise}
\label{sec:beyondbounded}
In this section we generalize the bounded noise assumption to the low noise assumption (a.k.a, Tsybakov noise)~\citep{mammen1999smooth,tsybakov2004optimal}, which is another standard assumption in the active learning literature. 
\begin{assum}[Low Noise]
\label{assum:low}
There exists constants $\alpha\in[0,1]$ and $c>0$, such that for any policy $\pi$, level $h\in[H]$, and $\epsilon>0$,
\[
P(|f^*_h (s_h,a_h)-1/2|\leq\epsilon | s_h, a_h \sim \pi )<c \cdot \epsilon^\alpha.
\]
\end{assum}
In this case, the difficulty of the reward learning problem depends on the exponent $\alpha$.

With this assumption, we can design algorithms with similar feedback and sample complexity. Concretely, We run Algorithm~\ref{alg:explore} and Algorithm~\ref{alg:plan} with $\Delta={(cK)^{-\frac{1}{\alpha}}}$, where $K$ is the number of episodes, $c$ and $\alpha$ are the constants in Assumption~\ref{assum:low}.
We state the theoretical guarantee in Theorem~\ref{thm:main_low}.
\begin{thm}
\label{thm:main_low}
In the linear case, under Assumption~\ref{assum:low}, Algorithm~\ref{alg:explore} and Algorithm~\ref{alg:plan} with $\Delta=(cK)^{-\frac{1}{\alpha}}$ can find an $\epsilon$-optimal policy with probability at least $1-\delta$, with at most
\[
K={O}\left(\frac{|\cA|^2d^5\dim^3(\cF)H^4\iota^3}{\epsilon^2}\right),\ \ \iota=\log(\frac{HSA}{\epsilon\delta})
\]
episodes. The total number of queries to the reward is bounded by
\[
\tilde{O}\left(H\cdot\dim^2(\cF)\cdot\left(\frac{|\cA|^4d^{10}\dim^6(\cF)H^8}{\epsilon^4}\right)^{\frac{1}{\alpha}}\right).
\]
In the tabular case, under Assumption~\ref{assum:low}, Algorithm~\ref{alg:explore} and Algorithm~\ref{alg:plan} with $\Delta=(cK)^{-\frac{1}{\alpha}}$can find an $\epsilon$-optimal policy with probability at least $1-\delta$, with at most
\[
O\left(\frac{H^4SA\iota}{\epsilon^2}+\frac{H^3S^2A\iota^2}{\epsilon}\right), \ \ \iota=\log(\frac{HSA}{\epsilon\delta})
\]
episodes. The total number of queries to the reward is bounded by 
\[
\tilde{O}\left(H\cdot\dim^2(\cF)\cdot\left(\frac{H^8S^2A^2}{\epsilon^4}+\frac{H^6S^4A^2}{\epsilon^2}\right)^{\frac{1}{\alpha}}\right).
\]
\end{thm}
\begin{proof}
We denote $\Delta=(cK)^{-\frac{1}{\alpha}}$ in the proof. Let $\epsilon=\Delta=(cK)^{-\frac{1}{\alpha}}$ in Assumption~\ref{assum:low}. We have that for any policy $\pi$, and level $h\in[H]$,
\[
P(|f^*_h (s_h,a_h)-1/2|\leq \Delta | s_h, a_h \sim \pi )<\frac{1}{K}.
\]
By a martingale version of the Chernoff bound, we have that with probability at least $1-\delta$, the number of elements in \[
\cG_h=\left\{(s_h^k,a_h^k)\Big| k\in[K], |f_h^*(s_h^k,a_h^k)-\frac12|\leq \Delta\right\}
\]
is at most $O(\log(H/\delta))$ for all $h\in[H]$.

The active learning algorithm guarantees to learn the reward labels correctly for all the elements in the dataset $\cD=\{(s_h^k,a_h^k)\}_{(h,k)\in[H]\times[K]}$, except the ones in $\bigcup_{h=1}
^H\cG_h$. Simply rehashing the proof of Theorem~\ref{thm:main} shows the optimality of the output policy. Plugging in the value of $\Delta$ to Theorem~\ref{thm:main} gives the bound on the total number of queries.
\end{proof}

\end{document}